\documentclass[11pt]{article}
\PassOptionsToPackage{hypertexnames=false}{hyperref}
\usepackage[preprint]{acl}
\usepackage{times}
\usepackage{latexsym}
\usepackage[T1]{fontenc}
\usepackage[utf8]{inputenc}
\usepackage{microtype}
\usepackage{inconsolata}
\usepackage{graphicx}
\usepackage{algorithm}
\usepackage{algorithmic}
\usepackage{amsmath}
\usepackage{amssymb}
\usepackage{amsthm}
\usepackage{multirow}
\usepackage{booktabs}
\usepackage{caption}

\newcommand{\ind}{\mathbf{1}}

\theoremstyle{plain}

\makeatletter
\@ifundefined{theorem}
  {\newtheorem{theorem}{Theorem}}
  {}

\@ifundefined{proposition}
  {\newtheorem{proposition}[theorem]{Proposition}}
  {}

\@ifundefined{lemma}
  {\newtheorem{lemma}[theorem]{Lemma}}
  {}

\@ifundefined{corollary}
  {\newtheorem{corollary}[theorem]{Corollary}}
  {}

\theoremstyle{remark}

\@ifundefined{remark}
  {}
  {}
\makeatother

\title{LLM-OSDA: An Optimal-Stopping Dynamic Auction for Native Advertising in Multi-Turn LLM Conversations}
\author{
  Yan Fang\textsuperscript{1}, Jialin Chen\textsuperscript{2,1},
  Chun Gan\textsuperscript{1\thanks{Corresponding author}}, Hang Yu\textsuperscript{3,1}, \\
  \textbf{Mingjun Nie\textsuperscript{1}, Yeyu Zhang\textsuperscript{2},
  Fengxiang He\textsuperscript{4}, Ching Law\textsuperscript{1}} \\
  \textsuperscript{1}JD.com, Inc. \\
  \textsuperscript{2}Shanghai University of Finance and Economics \\
  \textsuperscript{3}University of Sydney \quad
  \textsuperscript{4}University of Edinburgh
}

\begin{document}

\maketitle

\begin{abstract}
In multi-turn conversational advertising, insertion timing changes both the language context in which an ad appears and the click opportunities allocated to advertisers. We present LLM-OSDA, a single-insertion framework that jointly analyzes stopping, allocation, and payment through expected discounted clicks. The framework separates bid-free quality estimation and rendering from bid-dependent stopping, allocation, and payment. Under report-independent histories and click laws, exact causal Bellman optimization gives a monotone allocation; matching envelope payments ensure truthfulness in expectation. Separate conditional bounds characterize stopping-value and incentive losses under approximation. A learned implementation combines contextual click-quality estimation, recorded-trajectory StopNet training, native response generation, and critical-report pricing. On simulated test conversations, it improves proxy revenue over first-turn and final-turn insertion by 2.51\% and 12.00\%, respectively, using the same fine-tuned dialogue generation model. With the same dialogue generation model, an insertion policy based on a threshold for the current score earns higher mean revenue, while the reference StopNet checkpoint attains 0.86\% higher predicted welfare. Code is at \url{https://github.com/2025Fang2025/llm-osda}.

\end{abstract}

\section{Introduction}

{Large language models \citep{brown2020language,zhao2023survey} support new search interfaces \citep{metzler2021rethinking} and human-behavior simulation \citep{park2023generative}. Their deployment requires efficient inference \citep{pope2023efficiently}. LLM-native advertising embeds sponsored content in generated responses \citep{feizi2025online}. This shifts placement from fixed slots to conversational insertion moments (Figure~\ref{fig:ad-monetization-evolution}). As user intent evolves across dialogue turns, the platform must decide both what to show and when to insert it.}

\begin{figure}[t]
\centering
\includegraphics[width=0.8\columnwidth]{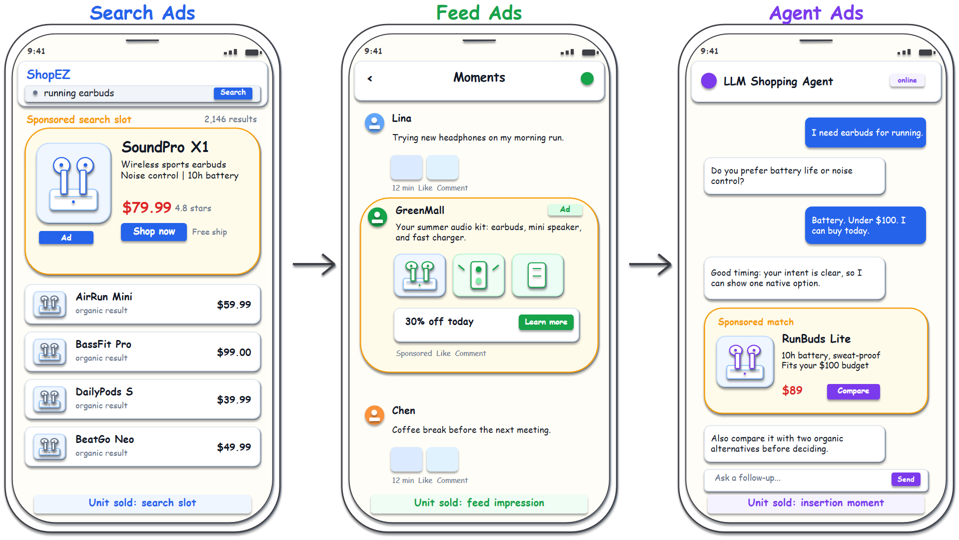}
\caption{Advertising monetization from search slots to feed impressions to agentic conversations, where timing becomes the decision.}
\label{fig:ad-monetization-evolution}
\end{figure}

{Existing generative auctions address token-level generation, retrieval-augmented responses, and advertiser placement in summaries \citep{duetting2024mechanism,hajiaghayi2024ad,dubey2024auctions}. Conversational auction timing and dynamic advertising mechanisms have also been studied \citep{banchio2024ads,alaei2026dynamic}. We investigate how to combine bid-dependent timing with contextual click-quality estimation and native response generation in a committed auction framework.}

{With one native insertion opportunity per session, showing an ad early may precede clear commercial intent, whereas waiting risks user departure. A higher bid can change both the winner and the insertion turn, affecting click probability and discounting. Even when the winner remains unchanged, a different insertion context can change its click opportunities. Incentives must therefore be analyzed over the complete stopping-and-allocation policy: a turn-local pricing argument does not suffice.}

{We propose the LLM-based Optimal Stopping Dynamic Auction (LLM-OSDA), a dynamic cost-per-click (CPC) framework with expected discounted clicks as its allocation measure. It separates bid-free quality estimation and response generation from bid-dependent stopping, allocation, and payment. The ideal mechanism uses causal Bellman optimization and envelope pricing; the implementation uses StopNet and critical-report payments. Our contributions are as follows:}

\begin{itemize}
\item {We introduce a learning-based framework that jointly decides when to insert an ad, which advertiser wins, how to generate the sponsored response, and how much to charge per click. LLM-OSDA combines contextual click-quality estimation and native response generation with a committed, bid-dependent stopping-and-allocation policy.}

\item {We construct an ideal CPC pricing mechanism for the complete stopping-and-allocation policy, accounting jointly for bid-induced changes in the winner and insertion time. Its envelope payment uses expected discounted-click allocation to price the timing option value that a single-turn local critical price generally cannot capture.}

\item {Under report-independent histories and click laws, exact Bellman optimization with matching envelope payments guarantees truthfulness in expectation and click-contingent individual rationality. We derive separate conditional bounds on stopping-value loss and misreport gains. Under the corresponding scoring and distributional assumptions, the single-turn case recovers welfare maximization or the quality-weighted Myerson auction.}

\item {On simulated conversations, LLM-OSDA improves proxy revenue by 2.51\% over first-turn insertion and 12.00\% over final-turn insertion using the same fine-tuned dialogue generation model. The reference StopNet checkpoint also achieves 0.86\% higher predicted welfare than a current-score threshold policy.}
\end{itemize}

\begin{figure*}[t]
  \centering
  \includegraphics[width=0.96\textwidth]{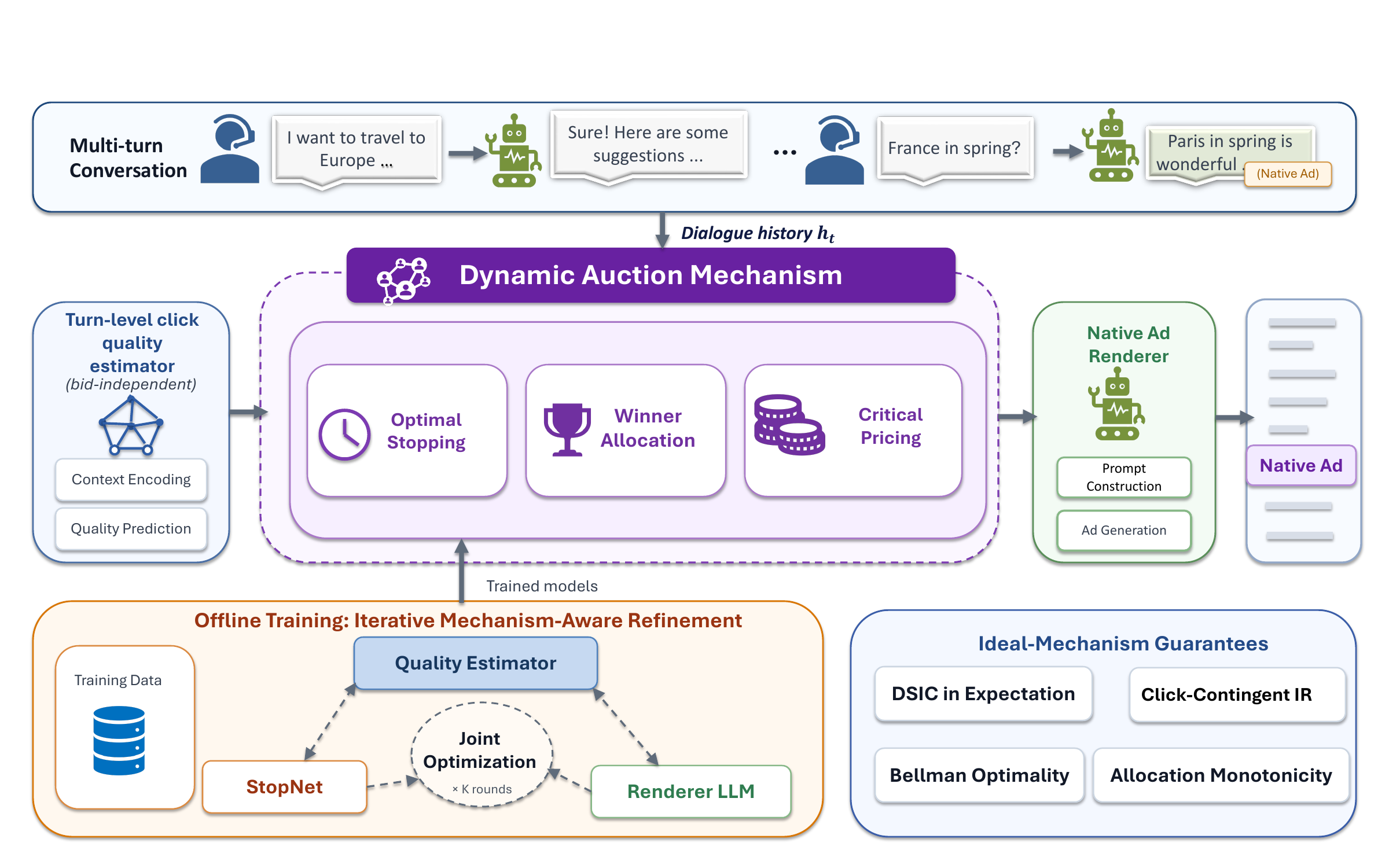}
  \caption{Overview of the LLM-OSDA framework.}
  \label{fig:llm-osda-framework}
  \end{figure*}
 
\section{Related Work}

{Myerson's virtual-value framework characterizes revenue-optimal auctions \citep{myerson1981optimal}. Position-auction research studies generalized second-price and truthful keyword pricing with exogenous advertising slots \citep{edelman2007internet,varian2007position,aggarwal2006truthful}; our allocation also depends on the insertion time. Related work examines budget-constrained auction learning \citep{balseiro2019learning}, differentiable auction design \citep{dutting2024optimal}, and permutation-equivariant architectures \citep{qin2022auction}.}

{Dynamic mechanism design studies incentives under evolving information \citep{bergemann2019dynamic,athey2013efficient,pavan2014dynamic}. Our analysis draws on single-parameter envelope methods \citep{myerson1981optimal,milgrom2002envelope} and optimal stopping \citep{peskir2006optimal}. Board's model of selling options connects mechanism design with the timing of option exercise \citep{board2007selling}. Here, by contrast, private values are fixed and only dialogue histories evolve.}

{LLM advertising connects market design, content generation, and ethics \citep{Du2024Advertis}. Existing approaches intervene at different stages of response production. At the generation and aggregation stages, token auctions link output aggregation to incentives \citep{duetting2024mechanism}, while response aggregation incorporates advertiser preferences \citep{soumalias2025truthful}. Auctions for generated summaries couple auction and LLM modules through a prediction model \citep{dubey2024auctions}. Building on retrieval-augmented generation \citep{lewis2020retrieval}, advertising approaches integrate ad selection with response production \citep{hajiaghayi2024ad,zhao2025llm}; LERA combines two-stage retrieval and ranking with critical-value payments \citep{sun2026lera}. PILA instead separates insertion from upstream generation through conditional response rewriting \citep{zhang2026pila}. Evaluation research connects simulated click intent with incentive-compatible payment design \citep{turnersmith2026evaluating} and examines advertiser relevance judgments \citep{dey2025judge}. Other related auction settings studied in the literature include sponsored questions \citep{bhawalkar2025sponsored} and data acquisition for RAG \citep{han2025data}.}

{Conversational recommendation uses evolving dialogue context to select relevant items \citep{friedman2023leveraging,he2023large}, with dedicated work on evaluation \citep{wang2023rethinking}. On the auction side, \citet{banchio2024ads} study platforms that can commit to the auction format but not its timing. \citet{alaei2026dynamic} study dynamic learning and advertising mechanisms on LLM platforms, with welfare guarantees and a virtual-value formulation for revenue maximization. Our focus is a committed stopping-and-allocation framework integrating contextual quality estimation, native rendering, and learned-policy incentive diagnostics.}

\section{Problem Formulation}
\label{sec:problem-formulation}

\subsection{Setting and Information Structure}
\label{subsec:problem-setup}

{A session has a given finite horizon $T$, at most one native ad, and a fixed per-turn discount $\gamma\in(0,1]$. There are $n$ advertisers indexed by $\mathcal N=\{1,\ldots,n\}$. Advertiser $i$ has a private per-click value $\theta_i\in\Theta_i=[\underline\theta_i,\overline\theta_i]$, where $0\leq\underline\theta_i<\overline\theta_i$. Before the session, each advertiser submits a fixed CPC bid $b_i\in\Theta_i$, forming the bid profile $b=(b_1,\ldots,b_n)$.}

{Let $H_t$ denote the random dialogue history, $h_t$ its realization, and $\mathcal F_t$ the available history information at turn $t$. The click probability $q_i(h_t)$ is defined under a fixed rendering procedure. The mechanism observes $(b,h_t)$ and determines the insertion time $\tau$, winner $I$, and CPC $P$. Decisions are causal: $\{\tau\leq t\}\in\mathcal F_t$ for fixed bids and horizon. No insertion is represented by $\tau=\infty$ and $I=P=0$.}

{The ideal analysis assumes report-independent history distributions, transitions, and conditional click probabilities. Bids only select among causal stopping-and-allocation plans in a common set $\Sigma_T$, including no insertion; neither the feasible set nor any fixed plan's mapping changes.}

\subsection{Allocation and Objective}
{For a plan $\sigma\in\Sigma_T$, advertiser $i$'s allocation is its expected discounted clicks:}
\begingroup
\[
x_i^\sigma=\mathbb E\!\left[
\gamma^{\tau^\sigma-1}q_i(H_{\tau^\sigma})
\ind\{\tau^\sigma<\infty,I^\sigma=i\}\right],
\]
\endgroup\ignorespacesafterend
{with zero integrand when no ad is inserted.}
{For nondecreasing scores $\psi_i$, the ideal mechanism chooses}
\begingroup
\[
\begin{gathered}
\sigma^\star(b)\in\arg\max_{\sigma\in\Sigma_T}J^\sigma(b),\\
J^\sigma(b)=\sum_i\psi_i(b_i)x_i^\sigma.
\end{gathered}
\]
\endgroup\ignorespacesafterend
{Write $x_i(b)=x_i^{\sigma^\star(b)}$.}
{With expected discounted payment $m_i(b)$, utility is $U_i(\theta_i;b)=\theta_i x_i(b)-m_i(b)$.}
{DSIC in expectation requires $U_i(\theta_i;\theta_i,b_{-i})\geq U_i(\theta_i;r_i,b_{-i})$ for every $i$, $\theta_i,r_i\in\Theta_i$, and fixed $b_{-i}$.}

{Experiments use $\psi_i(b_i)=b_i$ to optimize discounted bid-weighted clicks, equal to model welfare under truthful bidding; revenue and experience-penalized reward serve only as evaluation metrics.}

\section{Methodology}
\label{sec:methodology}
\subsection{Overview}
As shown in Figure~\ref{fig:llm-osda-framework}, LLM-OSDA predicts candidate ads' click quality from the dialogue history at each turn. It combines these predictions with bids to decide whether to wait or insert an ad. Upon stopping, the mechanism selects the winning advertiser and determines the CPC. The dialogue generation model then incorporates the selected ad into its response.

\subsection{Turn-level click quality estimator}
{We predict advertiser $i$'s click probability from dialogue history $h_t$, denoted by $\widehat q_i(h_t)$. The true click probability $q_i(h_t)$ is defined under a fixed, possibly stochastic response generation mechanism. Neither quality estimation nor response generation uses the bid.}

\subsection{Optimal stopping}
\paragraph{Ideal Bellman rule.}
{The ideal mechanism ranks advertisers by $\psi_i(b_i)q_i(h_t)$ for nondecreasing $\psi_i$.}
{At turn $t$, its immediate value and Bellman action values are}
\begingroup
\begin{align}
W_t(h_t;b)&=\max\{0,\max_i\psi_i(b_i)q_i(h_t)\},\notag\\
Q_t^S&=W_t,\notag\\
Q_t^W&=\gamma\mathbb E[V_{t+1}(H_{t+1};b)\mid H_t=h_t],\notag\\
V_t&=\max\{Q_t^S,Q_t^W\},\qquad V_{T+1}=0.
\end{align}
\endgroup\ignorespacesafterend

The ideal mechanism stops at the first turn whose immediate value is positive and at least the continuation value,
\begingroup
\[
\tau^\star(b)=\inf\{t\leq T:W_t>0,\ Q_t^S\geq Q_t^W\},
\]
\endgroup\ignorespacesafterend
with the convention $\inf\varnothing=\infty$.
{Ties favor stopping.}
{Since $Q_T^W=0$, the policy inserts an ad at the terminal turn if the score is positive, and otherwise returns no ad.}

\paragraph{Learned StopNet.}
{The ideal recursion requires the true click probabilities and history transitions.}
{StopNet predicts stop and wait values, $\widehat Q_t^S$ and $\widehat Q_t^W$, from the current state. Define the predicted immediate score as}
\begingroup
\[
\widehat W_t=\max\{0,\max_i\psi_i(b_i)\widehat q_i(H_t)\},
\]
\endgroup\ignorespacesafterend
{The implemented stopping rule is}
\begingroup
\[
\begin{aligned}
\widehat\tau(b)=\inf\{t\leq T:{}&\widehat W_t>0,\\
&[t=T\ \text{or}\ \widehat Q_t^S\geq\widehat Q_t^W]\}.
\end{aligned}
\]
\endgroup\ignorespacesafterend
{At the terminal turn, $\widehat W_T>0$ triggers stopping regardless of the network outputs.}

\subsection{Winner allocation}
The winner maximizes the ranking score times the click probability at the stopping turn.
The ideal policy takes \(I^\star(b)\in\arg\max_i\psi_i(b_i)q_i(H_{\tau^\star})\) with the true \(q_i\); the deployed policy uses the estimate,
\begingroup
\[
\widehat I(b)\in\arg\max_i\psi_i(b_i)\widehat q_i(H_{\widehat\tau}).
\]
\endgroup\ignorespacesafterend
{Ties follow the same fixed deterministic rule: the first maximum in the fixed candidate order wins.}
Because the winner is selected on the history \(H_{\tau^\star}\) at the stopping turn, who wins and when to insert are coupled: raising a bid can change both the winner and the trigger turn.
{Because the deployed policy ranks by the estimate \(\widehat q_i\), estimation error can shift both the trigger turn and the selected winner.}

\subsection{Ideal Envelope and Implemented Pricing}

{Fix $b_{-i}$ and write $x_i(b_i)=x_i(b_i,b_{-i})$. With zero truthful utility for the lowest type, envelope pricing \citep{myerson1981optimal} gives expected discounted payment $m_i$ and CPC $P_i$:}
\begingroup
\begin{align}
m_i(b_i)&:=b_ix_i(b_i)-\int_{\underline\theta_i}^{b_i}x_i(z)\,dz,\notag\\
P_i(b_i)&:=
\begin{cases}
m_i(b_i)/x_i(b_i),&x_i(b_i)>0,\\
0,&x_i(b_i)=0.
\end{cases}
\end{align}
\endgroup\ignorespacesafterend
{Payment and allocation use the same discount and click law. For fixed bids, ideal CPC is path-independent and charged only on a click.}

{The envelope payment prices the option value of bid-dependent timing. If $x_i$ is a monotone step function with jumps $\Delta x_{ik}$ at thresholds $c_{ik}\in(\underline\theta_i,b_i]$, then for $x_i(b_i)>0$,}
\begingroup
\[
P_i(b_i)=
\frac{
\underline\theta_i x_i(\underline\theta_i)
+\sum_{c_{ik}\le b_i}c_{ik}\Delta x_{ik}
}{x_i(b_i)}.
\]
\endgroup\ignorespacesafterend
{These thresholds capture changes in the full stopping-and-allocation policy; a single-turn local critical price generally cannot represent this payment.}

{The implementation uses critical-report CPC, computed by binary-search replay of the winner's bid. This need not equal envelope CPC, even for monotone click allocation. We separately report frozen-record diagnostics of numerical envelope payments in the main-text experiments; computation details in the supplementary material.}

\subsection{Iterative mechanism-aware refinement.}
{Because updating the renderer makes the labels used by the quality model and StopNet stale, we align the three learned components iteratively: the quality model \(G_\eta\), the StopNet \(\widehat Q_\phi=(\widehat Q_\phi^S,\widehat Q_\phi^W)\), and the renderer LLM \(\pi_R\). Each round regenerates responses with the current renderer, re-scores them with a user-simulation judge, and refits \(G_\eta\) and \(\widehat Q_\phi\) on updated labels. StopNet minimizes squared error to stop/wait targets computed backward along recorded trajectories with a fixed bid vector per record. The renderer is then refined by best-of-\(N\) sampling, keeping the response with the highest expected click value \(b_iq_i\). Each round freezes a committed tuple \((G_\eta^{(k)},\widehat Q_\phi^{(k)},\pi_R^{(k)})\). The target recursion and detailed training algorithms appear in the supplementary material.}

\begin{table*}[t]
\centering
\small
\setlength{\tabcolsep}{4pt}
\resizebox{\textwidth}{!}{%
\begin{tabular}{lrrrrrrrr}
\toprule
& \multicolumn{3}{c}{Platform objectives} & \multicolumn{5}{c}{Diagnostics} \\
\cmidrule(lr){2-4}\cmidrule(lr){5-9}
Method & Net Rev. & Reward & CTR & Trigger & Bid & CPC & Info. Rent & Bid Value \\
\midrule
Always-Round-1 (base) & $0.677_{\pm 0.000}$ & $0.637_{\pm 0.000}$ & $0.785_{\pm 0.000}$ & $1.000_{\pm 0.000}$ & $1.693_{\pm 0.000}$ & $0.885_{\pm 0.000}$ & $0.619_{\pm 0.000}$ & $1.297_{\pm 0.000}$ \\
Always-Round-1 (refined) & $0.707_{\pm 0.000}$ & $0.671_{\pm 0.000}$ & $0.812_{\pm 0.000}$ & $1.000_{\pm 0.000}$ & $1.693_{\pm 0.000}$ & $0.885_{\pm 0.000}$ & $0.648_{\pm 0.000}$ & $1.355_{\pm 0.000}$ \\
Always-Last-Round (base) & $0.612_{\pm 0.000}$ & $0.571_{\pm 0.000}$ & $0.755_{\pm 0.000}$ & $2.182_{\pm 0.000}$ & $1.609_{\pm 0.000}$ & $0.853_{\pm 0.000}$ & $0.531_{\pm 0.000}$ & $1.143_{\pm 0.000}$ \\
Always-Last-Round (refined) & $0.647_{\pm 0.000}$ & $0.610_{\pm 0.000}$ & $0.792_{\pm 0.000}$ & $2.182_{\pm 0.000}$ & $1.609_{\pm 0.000}$ & $0.853_{\pm 0.000}$ & $0.552_{\pm 0.000}$ & $1.199_{\pm 0.000}$ \\
Random-Round (base) & $0.645_{\pm 0.000}$ & $0.604_{\pm 0.000}$ & $0.773_{\pm 0.000}$ & $1.596_{\pm 0.000}$ & $1.643_{\pm 0.000}$ & $0.864_{\pm 0.000}$ & $0.574_{\pm 0.000}$ & $1.219_{\pm 0.000}$ \\
Random-Round (refined) & $0.679_{\pm 0.000}$ & $0.643_{\pm 0.000}$ & $0.809_{\pm 0.000}$ & $1.596_{\pm 0.000}$ & $1.643_{\pm 0.000}$ & $0.864_{\pm 0.000}$ & $0.599_{\pm 0.000}$ & $1.279_{\pm 0.000}$ \\
{Myopic-Threshold (base)} & {$0.697_{\pm 0.000}$} & {$0.656_{\pm 0.000}$} & {$0.776_{\pm 0.000}$} & {$1.111_{\pm 0.000}$} & {$1.701_{\pm 0.000}$} & {$0.918_{\pm 0.000}$} & {$0.599_{\pm 0.000}$} & {$1.295_{\pm 0.000}$} \\
{Myopic-Threshold (refined)} & {$\mathbf{0.729}_{\pm 0.000}$} & {$\mathbf{0.693}_{\pm 0.000}$} & {$0.806_{\pm 0.000}$} & {$1.111_{\pm 0.000}$} & {$1.701_{\pm 0.000}$} & {$0.918_{\pm 0.000}$} & {$0.627_{\pm 0.000}$} & {$1.356_{\pm 0.000}$} \\
\midrule
LLM-OSDA (base) & $0.694_{\pm 0.004}$ & $0.653_{\pm 0.004}$ & $0.785_{\pm 0.001}$ & $1.076_{\pm 0.023}$ & $1.690_{\pm 0.002}$ & $0.907_{\pm 0.006}$ & $0.599_{\pm 0.004}$ & $1.293_{\pm 0.000}$ \\
LLM-OSDA (refined) & $0.724_{\pm 0.005}$ & $0.688_{\pm 0.005}$ & $\mathbf{0.814}_{\pm 0.001}$ & $1.076_{\pm 0.023}$ & $1.690_{\pm 0.002}$ & $0.907_{\pm 0.006}$ & $0.627_{\pm 0.003}$ & $1.352_{\pm 0.001}$ \\
\bottomrule
\end{tabular}}
\caption{Performance and economic metrics of ad-insertion policies. Values are mean $\pm$ SD; bold marks the best Net Rev., Reward, and CTR.}
\label{tab:main}
\end{table*}

\section{Theoretical Guarantees}
\label{subsec:methodological-guarantees}

{The ideal analysis assumes report-independent histories and click laws and nondecreasing bid scores $\psi_i(b_i)$. Using only the score vector and a fixed tie rule, the mechanism selects an optimal policy from a set $\Sigma_T$ shared across bids; each policy uses only current and past dialogue information.}

{We first establish allocation monotonicity, which supports the envelope-payment guarantees.}

\begin{lemma}[Monotonicity of the ideal expected-click allocation]
\label{lem:exact-monotonicity}
{Under the stated assumptions, fixing $b_{-i}$, advertiser $i$'s ideal Bellman allocation is nondecreasing in $\psi_i(b_i)$, hence in $b_i$.}
\end{lemma}
\begin{proof}[Proof sketch]
{For a fixed plan $\sigma$, the objective is affine in the score: $J^\sigma(y)=A_i^\sigma+yx_i^\sigma$. Adding the optimality inequalities at $y_L<y_H$ gives $(y_H-y_L)(x_i^{\sigma_H}-x_i^{\sigma_L})\geq0$. Equal score vectors select the same plan.}
\end{proof}

\begin{theorem}[DSIC in expectation and click-contingent IR]
\label{thm:dsic-ir}
{A nonnegative monotone expected-click allocation, paired with its exact normalized envelope payment under the same click law, yields DSIC in expectation and click-contingent IR, with $0\leq P_i(b)\leq b_i$.}
\end{theorem}
\begin{proof}[Proof sketch]
{Fix $i,b_{-i}$ and suppress these indices. Substituting the envelope payment into $U(\theta;r)=\theta x(r)-m(r)$ gives}
\[
\begin{aligned}
&U(\theta;\theta)-U(\theta;r)\\
&\quad=\begin{cases}
\displaystyle\int_r^\theta[x(z)-x(r)]\,dz,&r<\theta,\\[1mm]
\displaystyle\int_\theta^r[x(r)-x(z)]\,dz,&r>\theta.
\end{cases}
\end{aligned}
\]
{Both expressions are nonnegative by monotonicity, proving DSIC in expectation. Also, $0\leq\int_{\underline\theta}^{b}x(z)\,dz\leq(b-\underline\theta)x(b)$ implies $0\leq m(b)\leq bx(b)$. Thus $0\leq P(b)\leq b$ when $x(b)>0$, with zero payment otherwise. At truth, charging only on a click gives click-contingent IR.}
\end{proof}

\begin{proposition}[Stopping-value loss bound]
\label{prop:stopping-value-loss}
{Let $\bar V_t$ be the value of a policy greedy with respect to its nonterminal stop/wait predictions, sharing the ideal policy's eligibility, stop rewards, tie rules, and terminal actions. If its nonterminal stop/wait predictions have summed uniform error $\varepsilon_k$ relative to causal action values, then}
\begingroup
\[
0\leq V_t-\bar V_t\leq\sum_{k=t}^{T-1}\gamma^{k-t}\varepsilon_k.
\]
\endgroup\ignorespacesafterend
\end{proposition}

\begin{proposition}[Conditional approximate-IC transfer]
\label{prop:approximate-ic}
{Under the same environment and click law, let $(x_i,m_i)$ be an ideal DSIC pair. If the implemented allocation and payment satisfy $\|x_i^{\mathrm{imp}}-x_i\|_\infty\le\delta_x$ and $\|m_i^{\mathrm{imp}}-m_i\|_\infty\le\delta_m$ uniformly over reports and rival bids, then}
\[
  \begin{aligned}
  &\sup_{r_i}\bigl[U_i^{\mathrm{imp}}(\theta_i;r_i,b_{-i})
  -U_i^{\mathrm{imp}}(\theta_i;\theta_i,b_{-i})\bigr]\\
  &\qquad\le 2\overline\theta_i\delta_x+2\delta_m.
  \end{aligned}
  \]
\end{proposition}
{Full proofs are provided in the appendix.}

\section{Experiments}
\label{sec:experiments}

{We evaluate the learned implementation through proxy revenue, predicted welfare, and finite-grid incentive diagnostics on simulated recorded conversations.}

\subsection{Experimental Setup}
\label{subsec:exp-setup}

\paragraph{Data.}
{We use simulated shopping dialogues of one to three turns with ads drawn from Amazon Reviews 2023 \citep{hou2024amazon}. After cleaning, the train/validation/test splits contain 9,527/1,354/2,741 records. Click labels come from a user-simulation judge \citep{zheng2023judging}. Records end at ad insertion, not natural user departure.}

\paragraph{Models and training.}
{The quality model $G_\eta$ uses Qwen3-Embedding-0.6B~\citep{zhang2025qwen3embedding} and a two-layer MLP; StopNet is a shared-trunk MLP with stop/wait heads. The dialogue generator $\pi_R$ is Qwen3-4B~\citep{yang2025qwen3}, fine-tuned on advertising dialogues. StopNet is trained with three seeds while other components remain fixed, using $\gamma=0.95$. Detailed settings appear in the appendix.}

\paragraph{Baselines.}
{Policies share auction components and are evaluated with base and fine-tuned dialogue models. Always-Round-1/Last-Round insert at the first/final recorded turn; Random-Round samples uniformly over turns.} {Myopic-Threshold inserts at the first turn with winning score $b_i\widehat q_{\eta,i}>0.4$.} {Hindsight maximizes predicted welfare over the full trajectory as an offline reference.}

\paragraph{Metrics.}
{Platform metrics are net revenue and experience-penalized reward. An independent User-LLM judge scores clicks and retention; revenue uses discounted click probabilities and CPC, while predicted welfare uses quality estimates from $G_\eta$. Full metric definitions appear in the appendix.}

\subsection{Main Results}
\label{subsec:exp-main}

{Table~\ref{tab:main} compares {five} timing policies under matched base/refined renderers. Hindsight reference metrics are reported in the appendix. Zero standard deviations in the table reflect rounding. The main table uses the implemented critical-report CPC; envelope payments are evaluated separately under the same learned allocation and judged-click outcomes.}

\paragraph{Revenue comparison.}
{With the same trained renderer, StopNet revenue is $0.724432\pm0.004534$ across three training seeds, versus 0.706720 for Always-Round-1 and 0.646787 for Always-Last-Round. The mean relative gains are 2.51\% and 12.00\%, with sample standard deviations of 0.64 and 0.70 percentage points, respectively. All three seeds improve over Always-Round-1, with gains ranging from 1.87\% to 3.15\%. The mean paired revenue difference is $0.017711\pm0.004534$ across seeds; these standard deviations are not confidence intervals. All 2,741 records receive an insertion, with mean insertion turn $1.076\pm0.023$ for StopNet versus 2.182 for Always-Last-Round. These are given-bid replay results, not equilibrium revenue estimates.}

{The following welfare, record-length, insertion-frequency, experience, and adaptive-baseline analyses use the reference checkpoint trained with seed 42. Their paired bootstrap intervals condition on that checkpoint.}

\paragraph{Adaptive timing comparison.}
Myopic-Threshold earns refined-renderer revenue 0.728724, versus the three-seed StopNet mean 0.724432. At the seed-42 checkpoint, StopNet revenue is 0.719922: the paired difference is $-0.008802$ (95\% CI $[-0.013124,-0.004383]$). In contrast, predicted welfare is 1.326171 for StopNet versus 1.314923 for Myopic, a 0.86\% gain (difference CI $[0.007532,0.015147]$). Thus the learned policy improves its predicted-welfare metric without dominating this adaptive baseline on revenue. The Myopic revenue lies within the three-seed StopNet range, 0.719922--0.729032. Table~\ref{tab:myopic} reports the matched timing-by-renderer comparisons.

\begin{table*}[t]
\centering\small
\begin{tabular}{llrrrr}
\toprule
Timing & Renderer & Revenue diff. & Paired 95\% CI & Predicted welfare & Winner pCTR\\
\midrule
Myopic  & base    & --- & ---                       & 1.314923 & .779880\\
StopNet & base    & $-.006917$ & $[-.010447,-.003258]$ & 1.326171 & .796680\\
Myopic  & refined & --- & ---                       & 1.314923 & .779880\\
StopNet & refined & $-.008802$ & $[-.013124,-.004383]$ & 1.326171 & .796680\\
\bottomrule
\end{tabular}
\caption{StopNet--Myopic comparisons. Revenue differences are StopNet minus Myopic, with paired 95\% bootstrap CIs.}
\label{tab:myopic}
\end{table*}

\paragraph{Welfare reference.}
{Predicted welfare is 1.317009 for Always-Round-1, 1.326171 for StopNet, and 1.349969 for Hindsight. StopNet improves over Always-Round-1 by 0.70\%, with paired difference CI $[0.005576,0.013168]$, closing 27.80\% of the same-record predicted-welfare gap to Hindsight. This fraction is descriptive, not a causal approximation guarantee.}

\paragraph{User-side experience.}
\label{par:exp-retention}
{The mean User-LLM experience score is 0.856013 for refined StopNet and 0.855425 for refined Always-Round-1. These close point estimates do not establish statistical equivalence or the absence of real-user retention loss.}

\subsection{Timing and Generation Analysis}
\label{subsec:exp-timing-generation}

\paragraph{Results by record length.}
{Revenue gains concentrate in two-turn records; the three-turn difference interval includes zero. Under the same renderer, all timing policies coincide on one-turn records because no timing choice is available. Restricting evaluation to $T\ge2$ yields a 2.05\% revenue gain over Always-Round-1, with paired-bootstrap CI $[1.351\%,2.821\%]$ for the relative improvement.}

\paragraph{Insertion timing.}
\label{par:exp-timing}
{StopNet inserts on turns one, two, and three in 2,588, 153, and zero test records, respectively. Although StopNet does not receive record length, the evaluator uses it for terminal handling, including forced insertion on eligible one-turn records. This distribution describes the learned policy under record-defined horizons.}

\paragraph{Renderer comparison.}
Holding the timing policy and payment procedure fixed, refined rendering raises StopNet revenue from $0.693948\pm0.004058$ to $0.724432\pm0.004534$ across seeds; the fixed and random policies also improve. In the seed-42 timing-by-renderer comparison, refining the renderer raises revenue by 4.60\% for Myopic and 4.37\% for StopNet. Switching from Myopic to StopNet instead changes revenue by $-0.99\%$ with base rendering and $-1.21\%$ with refined rendering. An unmatched renderer comparison therefore cannot isolate the benefit of learned timing.

Predicted welfare uses pre-render quality, whereas revenue uses response-sensitive click scores and CPC payments. At seed 42, StopNet selects winners with 2.15\% higher mean predicted click quality but 1.87\% lower mean CPC than Myopic. These aggregate differences describe the tradeoff; they do not isolate its cause, because revenue is the mean of per-record discounted click--payment products.

\subsection{Mechanism Diagnostics}
\label{subsec:exp-verification}

{We examine allocation monotonicity, incentives, and payments by replaying the original winner on each frozen record, holding rival bids fixed. These finite-grid checks characterize the learned implementation rather than certify the ideal mechanism's guarantees.}

\paragraph{Allocation monotonicity.}
\label{par:exp-monotonicity}
{Theorem~\ref{thm:dsic-ir} requires monotone expected discounted-click allocation and matching envelope payments. Using predicted quality and a 64-point grid on $[0.1,4.1]$ augmented with the original bid and in-range critical CPC, 2,622 of 2,741 records (95.66\%) pass the monotonicity check; 119 are nonmonotone. Fixed-timing baselines are monotone on the same grid. Figure~\ref{fig:allocation-monotonicity} illustrates both behaviors; the examples are sampled by class and do not reflect class frequencies. Projection residuals and sampling details appear in the supplementary material.}

\begin{figure}[t]
\centering
\includegraphics[width=\columnwidth]{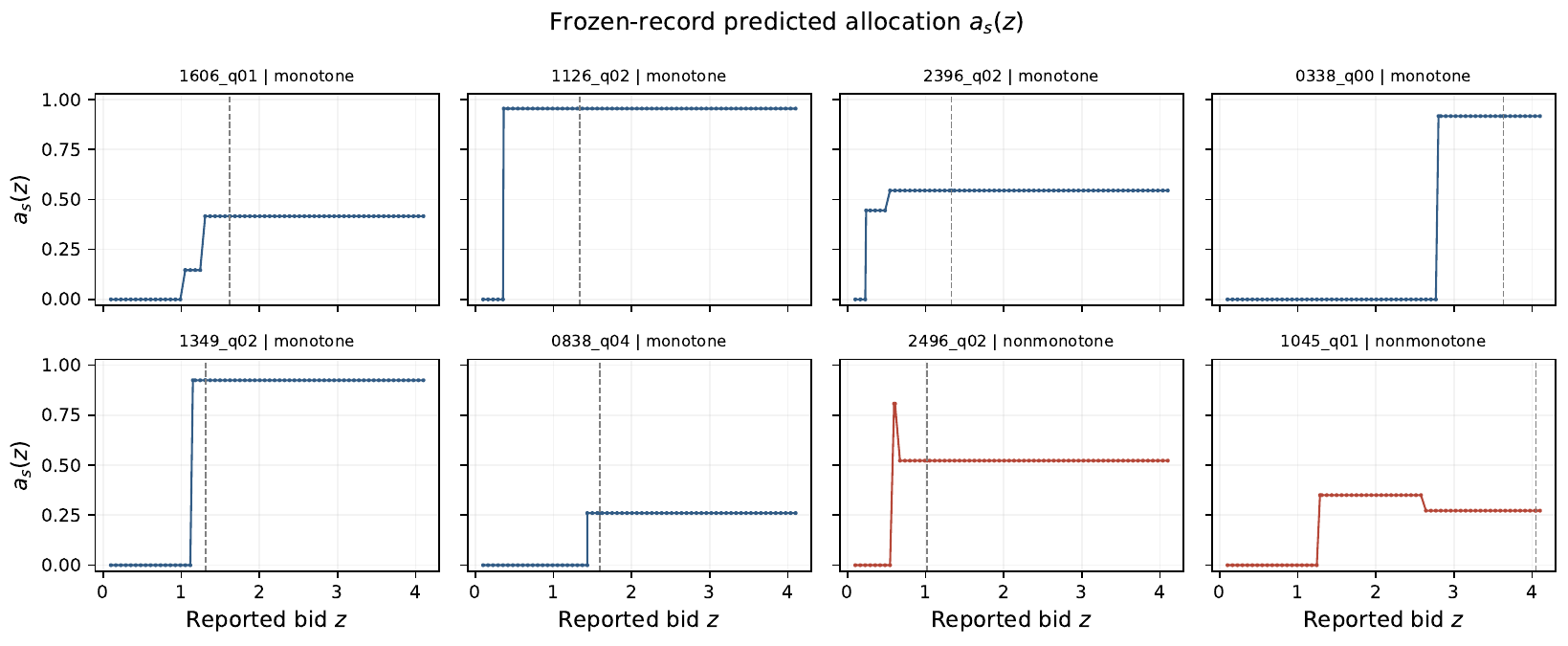}
\caption{{Predicted allocation under varying bids. Bottom right: two nonmonotone examples; dashed lines: original bids.}}
\label{fig:allocation-monotonicity}
\end{figure}

\paragraph{Finite-grid incentive diagnostics.}
\label{par:exp-ic}
{Using the original bid as an assumed private value, regret is the largest utility gain over tested reports under numerical-envelope transfers. At tolerance $10^{-9}$, 2,626 records (95.80\%) have numerical-zero regret and 115 (4.20\%) have positive regret (Figure~\ref{fig:ic-verification}). Mean regret is 0.006245; P95 is numerical zero, P99 is 0.217219, and the maximum is 0.726662. Fixed-timing baselines and all grid-monotone records have numerical-zero regret; the 119 nonmonotone records have mean regret 0.144.}

{Private values are assumed rather than observed. The diagnostic retains formal transfers even at zero allocation, so it does not measure incentives under the deployed critical-report CPC rule.}

\begin{figure}[t]
\centering
\includegraphics[width=\columnwidth]{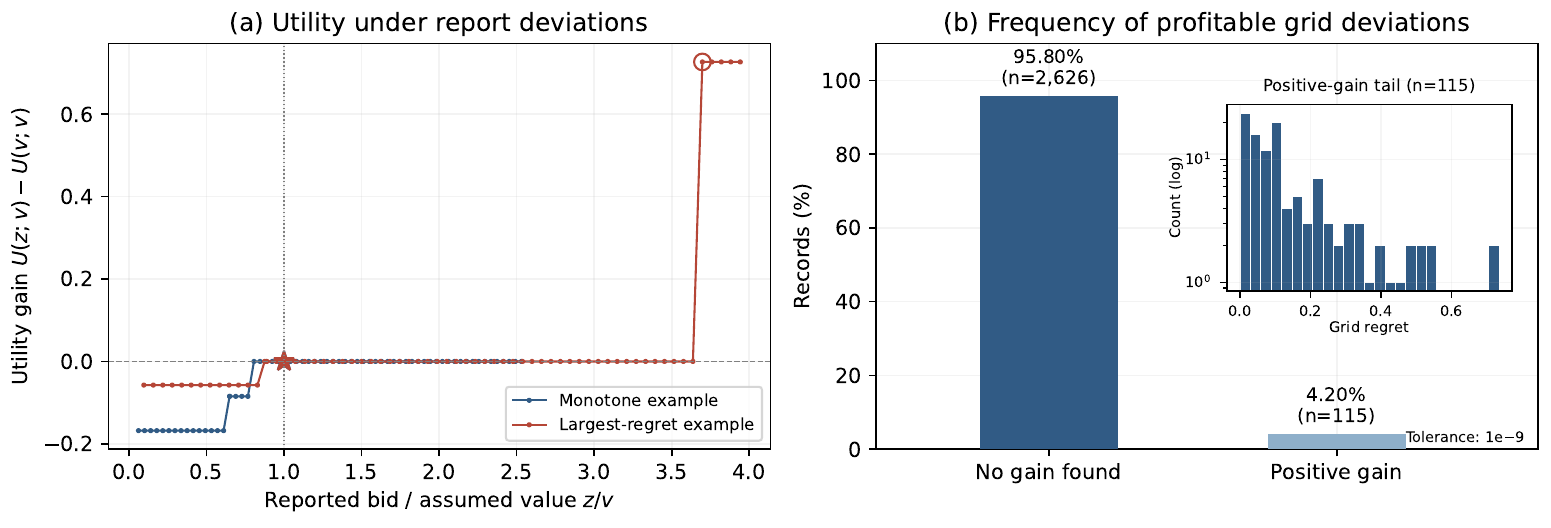}
\caption{{Numerical-envelope incentive diagnostics: (a) utility gains from misreporting; (b) zero/positive-regret shares, with the positive tail inset.}}
\label{fig:ic-verification}
\end{figure}

\paragraph{Envelope vs.\ critical-report payment.}
\label{par:exp-envelope}
{Envelope and critical-report CPC agree within $10^{-6}$ on 2,491 records (90.88\%; Figure~\ref{fig:envelope-vs-gdc}). The mean envelope-minus-critical gap is $+0.006404$ on grid-monotone records and $-0.056746$ otherwise. Multiple allocation increments and different lower bounds (0 versus 0.1) can cause payment gaps even with monotone allocation.}

{Holding allocation and click scores fixed, envelope payments yield a 2.17\% StopNet revenue gain over first-turn insertion with the refined renderer, with paired difference CI $[0.010618,0.020337]$. However, two records admit winning reports below the binary-search price, and one has a negative envelope payment. Changing the payment formula alone therefore does not restore the ideal guarantees. Further details appear in the supplementary material.}

\begin{figure}[t]
\centering
\includegraphics[width=\columnwidth]{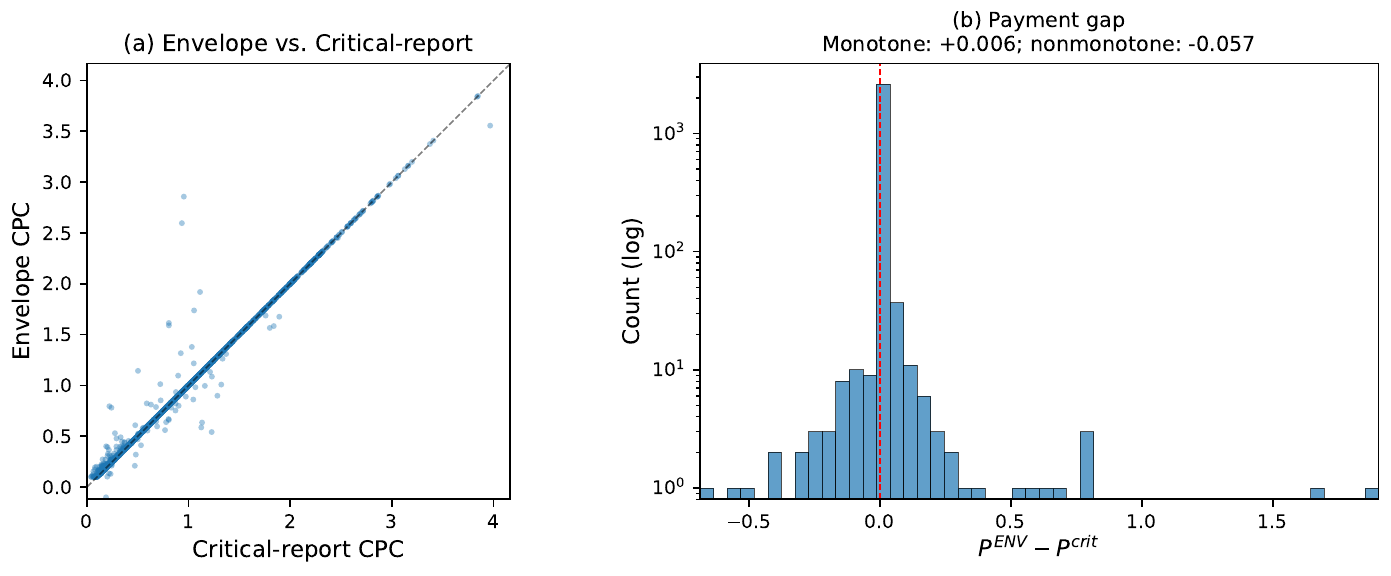}
\caption{{Envelope versus critical-report CPC: (a) payment comparison; (b) payment-gap distribution.}}
\label{fig:envelope-vs-gdc}
\end{figure}

\subsection{Ablation Study}
\label{subsec:component-analysis}

{We examine two design choices: the separation of mechanism components and the use of dialogue context in the click-quality estimator.}

\paragraph{Layering ablation.}
{Table~\ref{tab:layering-main} separately varies quality-model inputs, the stopping rule, and renderer inputs. Adding bids to $G_\eta$ and the renderer yields bid--score correlations of $|r|=0.58$ and $0.42$, respectively, indicating bid-sensitive component scores. Replacing the committed stopping comparison with a Qwen3-4B STOP/WAIT model gives only 18\% decision agreement with the original rule. This stopping comparison uses a separate reference configuration, not the main-table StopNet checkpoints.}

\begin{table}[t]
\centering
\small
\setlength{\tabcolsep}{3pt}
\begin{tabular}{@{}lccc@{}}
\toprule
Variant & pCTR $|r|$ & Rule agree. & Renderer $|r|$ \\
\midrule
Original reference & 0.00 & 1.00 & -- \\
bid-in-pCTR & 0.58 & -- & -- \\
llm-stop & -- & 0.18 & -- \\
bid-in-renderer & -- & -- & 0.42 \\
\bottomrule
\end{tabular}
\caption{{Layering ablations. $|r|$: Pearson correlation magnitude between $\log b$ and the component score; rule agree.: agreement with the original stopping rule.}}
\label{tab:layering-main}
\end{table}

\paragraph{Quality estimator ablation.}
{The full quality model achieves validation AUC 0.932, compared with 0.908 using only current-turn text and 0.689 using advertisement features alone, supporting the predictive value of dialogue history and context. These AUC comparisons measure click-prediction accuracy rather than end-to-end revenue.}

\section{Conclusion}
\label{sec:conclusion}

{LLM-OSDA combines bid-independent quality estimation and response generation with bid-dependent stopping, allocation, and payment for multi-turn native advertising. Under the stated assumptions, the ideal Bellman mechanism with matching envelope payments guarantees DSIC in expectation and click-contingent individual rationality. Conditional bounds quantify stopping-value and incentive losses under approximation.}

{In simulation with the same fine-tuned dialogue model, the learned implementation using critical-report payments improves mean proxy revenue by 2.51\% over first-turn insertion and 12.00\% over final-recorded-turn insertion. Compared with the adaptive Myopic-Threshold baseline, StopNet achieves 0.86\% higher predicted welfare, measured as discounted bid-weighted click value, but lower mean revenue. Future work will study incentive-compatible learning and evaluation in longer, real-user conversations.}

\section*{Limitations}

{Evaluation uses simulated shopping conversations of at most three turns from a single dataset, with User-LLM judgments rather than observed behavior. Generalization to real users, longer conversations, and other domains remains untested. During data synthesis, the dialogue generator could observe candidate bids, unlike the bid-free response generator used at deployment. Freezing these records for replay does not establish report-independent dialogue histories.}

{Exact DSIC and IR guarantees concern the ideal causal Bellman mechanism with matching envelope payments. Recorded-trajectory targets use realized continuations rather than causal conditional expectations. Finite-grid allocation and payment checks do not establish these guarantees for StopNet with critical-report payments or provide a full strategic-deviation or equilibrium evaluation.}

{The framework allows one insertion per session. The three-seed comparison varies only StopNet training, with other models fixed; bootstrap intervals measure test-record uncertainty conditional on the reference checkpoint. Neither analysis covers quality-model or User-LLM prediction errors, and full-system retraining variability remains unmeasured. The cause of the revenue shortfall relative to Myopic on three-turn records remains unresolved.}

\section*{Ethical Considerations}

{This study uses simulated user profiles, dialogues, clicks, and bids, with no real-user experiments. Synthetic data may inherit biases from product data and language models, and its results may not reflect real user preferences.}

{Sponsored content can affect recommendation impartiality. The response-generation prompt reproduced in the appendix explicitly prohibits mentioning advertising or sponsorship. Real deployment should replace this instruction with clear disclosure of paid content and provide user controls and product filtering. Simulated click and experience scores cannot replace independent evaluation of recommendation accuracy and effects on real users.}

\bibliography{custom}
\appendix

\section{Notation}
{Hats denote predicted quantities; $\star$ denotes the ideal optimizer.}
{For a fixed advertiser and rival report, we suppress $i,b_{-i}$ in $x,m,P,U$ and use $b,r,z$ for the submitted report, an alternative report, and the integration variable.}
{$\mathbb E$ denotes expectation and $\ind\{A\}$ the event indicator.}
{Model-parameter subscripts are omitted.}
{The same discount and click law are used when comparing expected allocations and payments.}
\begin{table*}[t]
\centering\small
\begin{tabular}{p{.25\textwidth}p{.68\textwidth}}
\toprule
Symbol & Meaning \\
\midrule
{ $n;i,j$} & { Number of advertisers; advertiser indices.} \\
{ $\theta_i,\Theta_i=[\underline\theta_i,\overline\theta_i]$} & { True value per click and the common type/report interval for advertiser $i$.} \\
{ $b=(b_i,b_{-i});r,z$} & { Submitted bids; an alternative report and an integration variable in a one-advertiser argument.} \\
{ $T,t,k;T_s,s,N$} & { Given horizon, current and future turns; length and index of a record, and number of test records.} \\
{ $H_t,h_t,\mathcal F_t;\omega$} & { Random history, its realization, available history information, and a potential path. For fixed bids, a causal stop satisfies $\{\tau\leq t\}\in\mathcal F_t$.} \\
{ $\gamma$} & { Per-turn discount, set to $0.95$ in the experiment; distinct from the experience proxy $g_s$.} \\
{ $q_i(h_t),\widehat q_i(h_t)$} & { True click probability and the quality-model prediction, the latter clipped to $[10^{-4},1]$.} \\
{ $\psi_i(b_i)$} & { Nondecreasing ranking score; the experiment uses the identity score.} \\
{ $\Sigma_T,\sigma,\sigma^\star(b),\bar\sigma$} & { Bid-independent set of causal plans; a fixed plan, exact optimizing plan, and auxiliary approximate stopping plan.} \\
{ $\tau^\star,\widehat\tau;I^\star,\widehat I$} & { Ideal/learned insertion times and winners. No insertion is represented by time $\infty$ and winner $0$.} \\
\bottomrule
\end{tabular}
\caption{{Unified notation (1). Hats denote predicted quantities; model-parameter subscripts are omitted.}}
\label{tab:notation}
\end{table*}

\begin{table*}[t]
\centering\small
\begin{tabular}{p{.25\textwidth}p{.68\textwidth}}
\toprule
Symbol & Meaning \\
\midrule
{ $x_i^\sigma,x_i,x_i^{\mathrm{imp}}$} & { Expected discounted clicks under a fixed plan, the ideal mechanism, and an implementation.} \\
{ $J^\sigma,A_i^\sigma,y$} & { Score-weighted objective; rivals' contribution for a fixed plan; local scalar score $y=\psi_i(b_i)$.} \\
{ $W_t,\widehat W_t$} & { Best nonnegative immediate score under true/predicted click probabilities.} \\
{ $Q_t^S,Q_t^W,V_t$} & { Exact STOP and WAIT action values and their maximum. Superscripts $S,W$ denote actions, not indices.} \\
{ $\widehat Q_t^S,\widehat Q_t^W,\bar V_t$} & { StopNet's action-value heads and the auxiliary stopping policy's value.} \\
{ $\varepsilon_t^S,\varepsilon_t^W,\varepsilon_t$} & { Uniform head-error bounds and their sum; defined only for nonterminal comparisons.} \\
{ $\Delta_t,\widehat\Delta_t,D_t$} & { Exact/predicted STOP-minus-WAIT margins and its supremum over relevant histories and bids.} \\
{ $Y_{s,t}^S,Y_{s,t}^W,L_{s,t}$} & { Recorded STOP/WAIT labels and their backward maximum; not causal Bellman ground truth.} \\
{ $m_i,P_i,U_i$} & { Expected discounted envelope transfer, corresponding CPC, and expected discounted utility.} \\
{ $m_i^{\mathrm{imp}},P_i^{\mathrm{imp}}(b;\omega)$} & { Implemented expected discounted transfer and possibly path-dependent CPC.} \\
\bottomrule
\end{tabular}
\caption{{Unified notation (2). Hats denote predicted quantities; model-parameter subscripts are omitted.}}
\end{table*}

\begin{table*}[t]
\centering\small
\begin{tabular}{p{.25\textwidth}p{.68\textwidth}}
\toprule
Symbol & Meaning \\
\midrule
{ $\delta_x,\delta_m$} & { Uniform implementation-to-ideal allocation and transfer error bounds.} \\
{ $c_i(h_1)$} & { Single-turn winning threshold at a fixed public state; not a click probability.} \\
{ $k,c_{ik},\Delta x_{ik}$} & { Local jump index, threshold, and click increment in the marginal-threshold calculation.} \\
{ $B_s^\sigma,B_s^{\mathrm{HS}},\overline B$} & { Realized predicted discounted score, full-record Hindsight maximum, and a sample mean of such scores.} \\
{ $I_s,\widehat\tau_s;\ell_s,g_s,d_s,P_s$} & { Record winner and insertion time; response-sensitive click proxy, experience proxy, discounted click proxy, and CPC.} \\
{ $\lambda,\overline g$} & { Experience-penalty weight ($0.25$) and mean experience score.} \\
{ $a_s(z),\widetilde m_s(z),\widetilde P_s(z)$} & { Frozen-record predicted-click allocation at report $z$, numerical envelope transfer, and its CPC. These are replay quantities, not population expectations.} \\
\bottomrule
\end{tabular}
\caption{{Unified notation (3). Hats denote predicted quantities; model-parameter subscripts are omitted.}}
\end{table*}

\section{Proofs of Theoretical Guarantees}
\label{app:methodology-proofs}

This appendix proves the theoretical guarantees stated in the main text.
The exact monotonicity and DSIC results apply to the ideal Bellman mechanism.

{Stopping stability applies to an auxiliary policy with common execution conditions; the error-to-IC transfer separately assumes uniform allocation and payment errors.}
Its finite-grid monotonicity and regret measurements are empirical diagnostics, not exact guarantees.

\subsection{Proof of Expected-Click Monotonicity}

\begin{proof}

{Fix advertiser \(i\) and \(b_{-i}\), and write $y=\psi_i(b_i)$.}

{Let $\Sigma_T$ be the common set of causal history-to-action plans, independent of the report being compared.}

{Fixing a plan fixes its mapping; changing the score only changes which plan is optimal.}

{Under report-independent transitions and clicks, for $\sigma\in\Sigma_T$, define}

\begingroup
\[
\begin{aligned}
x_i^\sigma&:=\mathbb{E}\!\left[\gamma^{\tau^\sigma-1}q_i(H_{\tau^\sigma})\ind\{\tau^\sigma<\infty,I^\sigma=i\}\right],\\
A_i^\sigma&:=\mathbb{E}\!\left[\begin{aligned}
&\sum_{j\neq i}\gamma^{\tau^\sigma-1}\psi_j(b_j)q_j(H_{\tau^\sigma})\\
&\qquad\cdot\ind\{\tau^\sigma<\infty,I^\sigma=j\}
\end{aligned}\right].
\end{aligned}
\]
\endgroup\ignorespacesafterend
{Both integrands are defined as zero when no ad is inserted.}

{For fixed \(\sigma\), both terms are independent of \(y\), and the objective is \(J^\sigma(y)=A_i^\sigma+yx_i^\sigma\).}

{Take \(y_L<y_H\), and let \(\sigma_L,\sigma_H\) be optimal at the two scores.}
Optimality gives

\begingroup
\begin{align*}
A_i^{\sigma_L}+y_Lx_i^{\sigma_L}&\ge A_i^{\sigma_H}+y_Lx_i^{\sigma_H},\\
A_i^{\sigma_H}+y_Hx_i^{\sigma_H}&\ge A_i^{\sigma_L}+y_Hx_i^{\sigma_L}.
\end{align*}
\endgroup\ignorespacesafterend
{Adding yields \((y_H-y_L)(x_i^{\sigma_H}-x_i^{\sigma_L})\ge0\).}

{Thus allocation is nondecreasing in \(y\).}
{If two reports have the same score, the full score vector is unchanged and the fixed selector chooses the same plan.}

{Since $\psi_i$ is nondecreasing, allocation is also nondecreasing in $b_i$.}
{This is an expected-allocation result, not a claim about every realized path.}
\end{proof}

\subsection{Proof of Envelope DSIC and IR}
\label{app:proof-dsic-ir}

\begin{proof}

{Fix $i,b_{-i}$ and suppress them; write $[\underline\theta,\overline\theta]$ for the type interval and $x(r)$ for expected discounted clicks.}
{The allocation and transfer use the same click law, and the lowest type has zero truthful utility.}

{Under the envelope payment \(m(r)=rx(r)-\int_{\underline\theta}^{r}x(z)\,dz\), type \(\theta\) has}
\[
\begin{aligned}
&U(\theta;\theta)-U(\theta;r)\\
&\quad=\begin{cases}
\displaystyle\int_r^\theta[x(z)-x(r)]\,dz, & r<\theta,\\[2mm]
\displaystyle\int_\theta^r[x(r)-x(z)]\,dz, & r>\theta.
\end{cases}
\end{aligned}
\]
Monotonicity of \(x\) makes both expressions nonnegative, so truthful reporting is weakly dominant in expectation over the bid-independent dialogue uncertainty.

{At truth, $U(\theta;\theta)=\int_{\underline\theta}^{\theta}x(z)\,dz\geq0$.}
{Moreover,}
\begingroup
\[
0\leq\int_{\underline\theta}^{b}x(z)\,dz\leq(b-\underline\theta)x(b),
\]
\endgroup\ignorespacesafterend
{so $\underline\theta x(b)\leq m(b)\leq bx(b)$.}
{For $x(b)>0$, division gives $0\leq P(b)\leq b$; for $x(b)=0$, monotonicity makes the integral and transfer zero.}
{Charging the fixed CPC only on a click realizes $m=Px$.}

{A realized truthful click therefore gives nonnegative utility, while no click incurs no charge.}

\end{proof}

\subsection{Conditional Stopping Stability}
\label{app:stopping-stability}

\begingroup
\begin{lemma}[Conditional stopping stability]
\label{lem:stopping-stability}
Consider an auxiliary policy $\bar\sigma$ that differs from the ideal policy only in its nonterminal WAIT/STOP comparison. Both policies have the same eligibility, stop reward $W_t$, tie rule, and terminal action. Suppose, uniformly over relevant histories and bids, for $t<T$,
\[
|\widehat Q_t^S-Q_t^S|\leq\varepsilon_t^S,\qquad
|\widehat Q_t^W-Q_t^W|\leq\varepsilon_t^W,
\]
and put $\varepsilon_t=\varepsilon_t^S+\varepsilon_t^W$. At a common eligible nonterminal state, disagreement implies $|Q_t^S-Q_t^W|\leq\varepsilon_t$. Its value $\bar V_t$ satisfies
\begin{equation}
0\leq V_t-\bar V_t\leq\sum_{k=t}^{T-1}\gamma^{k-t}\varepsilon_k.
\label{eq:stability}
\end{equation}
\end{lemma}
\endgroup\ignorespacesafterend
{Both policies share eligibility, stop rewards, and terminal handling; winner and eligibility errors require separate control. The error bounds compare predictions with causal action values, not recorded-trajectory targets.}

\begingroup
\begin{proof}
At an eligible nonterminal state put $\Delta_t=Q_t^S-Q_t^W$ and $\widehat\Delta_t=\widehat Q_t^S-\widehat Q_t^W$. The assumed head errors give
\[
|\widehat\Delta_t-\Delta_t|\leq\varepsilon_t^S+\varepsilon_t^W=\varepsilon_t.
\]
With the same tie rule, opposite decisions require $|\Delta_t|\leq\varepsilon_t$. If $\bar\sigma$ stops, its loss is zero when STOP is ideal, and at most $\varepsilon_t$ otherwise. If it waits, its immediate loss relative to the better exact action is at most $\varepsilon_t$, and subsequent decisions contribute at most $\gamma D_{t+1}$, where
\[
D_t=\sup_{h_t,b}\{V_t(h_t;b)-\bar V_t(h_t;b)\}.
\]
At an ineligible state both policies wait, or return no ad at the horizon, so there is no immediate decision loss. Terminal actions and rewards coincide, giving $D_T=0$. Thus $D_t\leq\varepsilon_t+\gamma D_{t+1}$. Backward induction yields
\[
0\leq V_t-\bar V_t\leq D_t\leq\sum_{k=t}^{T-1}\gamma^{k-t}\varepsilon_k.
\]
If the learned winner or eligibility differs, its stop reward need not equal $W_t$; that additional discrepancy is not covered by this induction. 
\end{proof}
\endgroup\ignorespacesafterend

\subsection{Proof of Approximate Incentive Compatibility}
\label{app:approximate-ic}

\begin{proof}
Let \(U(\theta;r)=\theta x(r)-m(r)\) and \(U^{\mathrm{imp}}(\theta;r)=\theta x^{\mathrm{imp}}(r)-m^{\mathrm{imp}}(r)\). Assume uniformly over reports and rival bids that $|x^{\mathrm{imp}}-x|\leq\delta_x$ and $|m^{\mathrm{imp}}-m|\leq\delta_m$, with $0\leq\theta\leq\overline\theta_i$. For any report \(r\), ideal DSIC and the triangle inequality give
\begin{align*}
&U^{\mathrm{imp}}(\theta;r)-U^{\mathrm{imp}}(\theta;\theta)\\
&\quad\le\theta|x^{\mathrm{imp}}(r)-x(r)|\\
&\qquad+\theta|x^{\mathrm{imp}}(\theta)-x(\theta)|\\
&\qquad+|m^{\mathrm{imp}}(r)-m(r)|\\
&\qquad+|m^{\mathrm{imp}}(\theta)-m(\theta)|\\
&\quad\le 2\overline\theta_i\delta_x+2\delta_m.
\end{align*}
\end{proof}

{The uniform bounds are evaluated under the same environment and click law for all reports and fixed rivals.}
{For a path-dependent implemented CPC,}
\begingroup
\begin{align*}
m_i^{\mathrm{imp}}(b)=\mathbb E\big[&\gamma^{\widehat\tau-1}q_i(H_{\widehat\tau})
\ind\{\widehat\tau<\infty,\widehat I=i\}\\
&\cdot P_i^{\mathrm{imp}}(b;\omega)\big],
\end{align*}
\endgroup\ignorespacesafterend
{with zero contribution for no insertion.}
{Mean price times mean allocation is not generally this expected transfer.}
{A different pricing rule, including critical-report pricing, and numerical payment errors both contribute to $\delta_m$.}
{The result bounds expected misreport gain; it does not separately guarantee nonnegative prices or IR.}
{Neither stopping-value loss nor training MSE establishes its uniform assumptions. Finite-grid diagnostics cannot verify these uniform-error conditions.}

\subsection{Marginal-Threshold Interpretation}
\label{app:proof-marginal-threshold}

\begin{lemma}[Global marginal-threshold decomposition]
\label{lem:marginal-threshold}

{Suppose \(x_i\) is a nondecreasing step function with jumps \(\Delta x_{ik}\) at report thresholds \(c_{ik}\in(\underline\theta_i,b_i]\).}
For any \(b_i\) with \(x_i(b_i)>0\), its envelope CPC is

\begingroup
\[
P_i(b_i)
=
\frac{
\underline\theta_i x_i(\underline\theta_i)
+\sum_{c_{ik}\le b_i}c_{ik}\Delta x_{ik}
}{x_i(b_i)}.
\]
\endgroup\ignorespacesafterend
Thus each threshold is weighted by the expected-click increment that it creates across the entire history-contingent mechanism.
\end{lemma}

\begin{proof}

{Write \(x_i(b_i)=x_i(\underline\theta_i)+\sum_{c_{ik}\le b_i}\Delta x_{ik}\).}
Integrating this step function and substituting it into the envelope identity gives

\begingroup
\begin{align*}
m_i(b_i)
&=b_ix_i(b_i)-\int_{\underline\theta_i}^{b_i}x_i(z)\,dz\\
&=\underline\theta_i x_i(\underline\theta_i)
+\sum_{c_{ik}\le b_i}c_{ik}\Delta x_{ik}.
\end{align*}
\endgroup\ignorespacesafterend
Dividing by \(x_i(b_i)\) proves the result.
The thresholds are global because changing a report may alter both the stopping time and the winner across dialogue histories.
\end{proof}

{If \(x_i(\underline\theta_i)=0\) and the allocation curve has one jump at \(c_i\), the lemma reduces to \(P_i=c_i\).}
This single-jump case justifies a critical-report interpretation.
Multiple timing-induced jumps generally do not admit one local critical-value formula.

{For a concrete analytical example on $[0.1,4.1]$, let}
\begingroup
\[
x(b)=\begin{cases}0,&b<1,\\0.2,&1\leq b<2,\\0.8,&b\geq2.\end{cases}
\]
\endgroup\ignorespacesafterend
{At $b=3$, the envelope CPC is $(1\cdot0.2+2\cdot0.6)/0.8=1.75$, whereas first-win pricing charges 1.}
{A type $\theta=1.5$ receives utility $(1.5-1)0.2=0.1$ at truth under first-win pricing, but $(1.5-1)0.8=0.4$ by reporting 2.}
{This counterexample concerns the payment rule even though allocation is monotone; it is not an observed experimental record.}

\subsection{Proof of Corollary~\ref{cor:t-one} (\texorpdfstring{$T=1$}{T=1} Special Case)}
\label{app:proof-t-one}

\begin{corollary}[\(T=1\) special case]
\label{cor:t-one}

{When \(T=1\), identity scoring yields a quality-weighted welfare auction, and a regular Myerson virtual-value score yields the corresponding quality-weighted Myerson auction under independent private values with known regular distributions and public qualities independent of those values.}
\end{corollary}

\begingroup
\begin{proof}
At $T=1$, $Q_1^W=0$. The mechanism chooses a positive maximizer of $\psi_i(b_i)q_i(h_1)$, or no ad if all scores are nonpositive. Identity scores maximize click-weighted welfare. With independent private values, known regular distributions, and public qualities independent of those values, nondecreasing Myerson virtual values yield the corresponding revenue-optimal single-turn allocation.

At fixed public state $h_1$ and rival bids, the winning set is an upper interval whenever advertiser $i$ can win. Let its lower threshold be $c_i(h_1)$, with $c_i(h_1)=\underline\theta_i$ when the lowest report wins. For positive quality, the conditional allocation is a single step of height $q_i(h_1)$, so the envelope CPC equals $c_i(h_1)$ on winning reports. Threshold-boundary ties follow the fixed rule and do not affect the integral. An advertiser that never wins pays zero.

If $H_1$ is random, the ex ante allocation may contain multiple steps. Its fixed envelope CPC is
\[
P_i(b)=\frac{\mathbb E[c_i(H_1)x_i(b\mid H_1)]}{\mathbb E[x_i(b\mid H_1)]}
\]
when the denominator is positive, with zero contribution from nonwinning states. Charging the conditional threshold state by state gives the same expected transfer but need not give the same realized CPC. Consequently the single-threshold interpretation is conditional on the public state.
\end{proof}
\endgroup\ignorespacesafterend

\section{Implementation and Additional Results}

{The main revenue table reports means and sample standard deviations over StopNet training seeds 42, 43, and 44, with the quality/judge model, renderer, split, target-bid sampling, and evaluation seed fixed. Unless explicitly stated otherwise, the detailed analyses, confidence intervals, and case study in this supplement use the seed-42 reference checkpoint. Their single-checkpoint statistics are not three-seed averages.}

\subsection{Iterative Refinement}

\begin{algorithm}[h]
\caption{Iterative Mechanism-Aware Refinement}
\label{alg:refine}
\begin{algorithmic}[1]
\STATE Fit \(G_\eta^{(1)},\widehat Q_\phi^{S,(1)},\widehat Q_\phi^{W,(1)},\pi_R^{(1)}\) from initial rollouts \(\mathcal D_0\); set \(\mathcal D_1\gets\mathcal D_0\).
\FOR{round \(k=2,3,\ldots,K\)}
  \STATE Initialize \(\mathcal D_k \gets \emptyset\).
  \FOR{each dialogue \(d\in\mathcal D_{k-1}\)}
    \STATE \(\tilde y_d \gets \pi_R^{(k-1)}(d)\) \hfill\COMMENT{regenerate response}
    \STATE \((\tilde c_d,\tilde\gamma_d) \gets \text{Judge}(d,\tilde y_d)\) \hfill\COMMENT{re-score click/retention}
    \STATE \(\{y_d^{(1)},\ldots,y_d^{(N)}\} \gets \pi_R^{(k-1)}(d;N)\) \hfill\COMMENT{best-of-\(N\) sampling}
    \STATE Score each candidate with the user-simulation judge: \(q_d^{(n)}\gets\text{Judge}_{\mathrm{click}}(d,y_d^{(n)})\).
    \STATE \(n^\star\gets\arg\max_n b_i q_d^{(n)}\); \(y_d^\star\gets y_d^{(n^\star)}\). \hfill\COMMENT{expected click value}
    \STATE \(\mathcal D_k \gets \mathcal D_k \cup \{(d,\tilde y_d,\tilde c_d,\tilde\gamma_d,y_d^\star)\}\)
  \ENDFOR
  \STATE Fit \(G_\eta^{(k)}\) with \(\mathcal L_{\mathrm{BCE}}\) and both \(\widehat Q_\phi^{S,(k)},\widehat Q_\phi^{W,(k)}\) with recorded-trajectory stop/wait-target \(\mathcal L_{\mathrm{MSE}}\) on \(\mathcal D_k\) (Algorithm~\ref{alg:stopnet-training}).
  \STATE \(\pi_R^{(k)} \gets \text{SFT}(\pi_R^{(k-1)};\{y_d^\star\}_{d\in\mathcal D_{k-1}})\).
\ENDFOR
\RETURN \((G_\eta^{(K)},\widehat Q_\phi^{(K)},\pi_R^{(K)})\)
\end{algorithmic}
\end{algorithm}

\begingroup
\subsection{Recorded-Trajectory StopNet Training}

\label{app:refine}
Algorithm~\ref{alg:refine} describes the iterative refinement framework; the reported experiments freeze the quality model and renderer from its second round and retrain StopNet using recorded-trajectory supervision. The reported comparisons hold those reused components fixed across StopNet seeds. Algorithm~\ref{alg:stopnet-training} specifies the recorded-trajectory labels used to train StopNet.
The experiment freezes the existing Qwen3-Embedding-0.6B encoder with its two-layer MLP quality predictor and the previously fine-tuned renderer.
StopNet is retrained once on all retained training lengths.
Each record uses one fixed bid vector while its turn-level quality scores and 12-dimensional features are constructed.
Backward labels are computed from the entire recorded trajectory, followed by supervised regression to WAIT and AUCTION outputs.
The execution rule still checks score eligibility before applying the terminal override.

\begin{algorithm}[t]
\captionsetup{font=10pt}
\caption{Recorded-Trajectory StopNet Supervision}
\label{alg:stopnet-training}
\begin{algorithmic}[1]
\FOR{each training record $s$}
\STATE Fix one bid vector for all $T_s$ turns.
\STATE Compute clipped qualities and $\widehat W_t(h_{s,t})$.
\STATE Set $L_{s,T_s+1}=0$.
\FOR{$t=T_s,T_s-1,\ldots,1$}
\STATE $Y_{s,t}^S\gets\widehat W_t(h_{s,t})$; $Y_{s,t}^W\gets\gamma L_{s,t+1}$.
\STATE $L_{s,t}\gets\max\{Y_{s,t}^S,Y_{s,t}^W\}$.
\STATE Store current features and the two action labels.
\ENDFOR
\ENDFOR
\STATE Fit StopNet by mean squared error on the stored labels.
\end{algorithmic}
\end{algorithm}
\endgroup\ignorespacesafterend

\paragraph{{Causal and recorded targets.}}
{Consider an analytical three-turn example with $\gamma=0.95$, $W_1=1.6$, $W_2=1$, and $W_3\in\{0,3\}$ with equal probabilities, unobserved at turn two.}
{The causal first-turn WAIT value is}
\begingroup
\[
\begin{aligned}
Q_1^W&=0.95\max\{1,0.95\mathbb E[W_3\mid H_2]\}\\
&=1.35375.
\end{aligned}
\]
\endgroup\ignorespacesafterend
{The mean recorded WAIT label is instead}
\begingroup
\[
\mathbb E[Y_{s,1}^W]=0.95\mathbb E[\max\{1,0.95W_3\}]=1.82875.
\]
\endgroup\ignorespacesafterend
{Thus the causal rule stops at turn one, whereas comparison with the mean recorded label favors waiting.}
{This illustrates the order of maximization and conditional expectation, not a measured training failure.}
{The labels define a learning heuristic and do not identify causal Bellman values merely through small MSE.}

\subsection{Data Generation}

{The original corpus contains 14,918 recorded dialogues of lengths one to three from $3{,}000$ user profiles.}
{Filtering retains 13,622 records with unchanged split membership: 9,527 training, 1,354 validation, and 2,741 test.}
{The test length counts are 208/1,825/708.}
{Generation stops after an insertion, so the recorded length is not a verified natural departure time.}
The ad pool has $100$ products across $50$ categories ($2$ per category), retrieved from the Amazon Reviews 2023 corpus~\cite{hou2024amazon}.
Each session draws three candidate ads: two retrieved for relevance to the profile and one random.
Dialogues are synthesized with proprietary GPT-family and Claude-family chat models (a user model, an assistant model, and a counterfactual generator), and click and retention labels come from a proprietary GPT-family judge under the LLM-as-a-judge paradigm~\cite{zheng2023judging}.
During synthesis the assistant may see candidate bids, reflecting a platform's tilt toward higher-paying ads, while the judge that assigns click and retention labels does not see the bid.

{Envelope diagnostics use report interval $[0.1,4.1]$ and treat the original bid as the reference value for regret. This is a diagnostic convention, not observation of private values. Reported base bids stored with the ad pool range over $[0.14,4.05]$ (mean $1.30$, right-skewed). Training augmentation is not clipped to the diagnostic interval.}

{For StopNet training, each record receives one bid draw per advertiser, $b_i\sim\mathrm{Unif}[\max(0.01,0.5b_i^{\mathrm{base}}),1.5b_i^{\mathrm{base}}]$, held fixed across its turns. Evaluation uses stored candidate bids; only payment and incentive probes vary reports.}

\subsection{Metric Definitions}\label{app:metrics}

{For test record $s$, let $I_s$ and $\widehat\tau_s$ be its winner and insertion time, $\ell_s\in[0,1]$ be the response-sensitive click probability, $g_s\in[0,1]$ the retention/experience score, and $P_s$ the CPC. Both $\ell_s$ and $g_s$ are scored by a separate User-LLM judge, not by the pCTR model $G_\eta$.}

{Define $d_s=\gamma^{\widehat\tau_s-1}\ell_s$ for an insertion, and zero otherwise.}
{On $N$ records,}

\begingroup
\begin{align*}
\mathrm{NetRev}&=N^{-1}\sum_s P_s d_s,\\
\mathrm{InfoRent}&=N^{-1}\sum_s(b_{I_s}-P_s)d_s,\\
\mathrm{BidValue}&=N^{-1}\sum_s b_{I_s}d_s
\\
&=\mathrm{NetRev}+\mathrm{InfoRent},\\
\mathrm{Reward}&=\mathrm{NetRev}-\lambda(1-\overline g),\\
&\quad\lambda=0.25.
\end{align*}
\endgroup\ignorespacesafterend
{For no insertion, payment and bid contributions are zero and the experience score is 1.}

{CTR averages undiscounted $\ell_s$ over inserted records; Trigger is the mean insertion turn among inserted records, and Bid and Pay likewise average their corresponding values over insertions. Revenue instead averages the per-record product $P_s\gamma^{\widehat\tau_s-1}\ell_s$, so multiplying mean Pay by mean CTR does not generally recover mean revenue.}
{All records in the main table have an insertion.}
{The click proxy is a probability score rather than a realized binary click.}
{Bid Value (called Gross Revenue in legacy logs) is discounted bid-weighted click value, not an additional payment collected by the platform.}
{With truthful bids, Information Rent is a proxy for advertiser utility.}
{The proxy score $g_s$ is distinct from the fixed discount $\gamma$.}

{For a policy $\sigma$, predicted welfare uses the pre-render quality:}
\begingroup
\[
B_s^\sigma=\gamma^{\tau_s^\sigma-1}b_{I_s^\sigma}\widehat q_{I_s^\sigma}(h_{s,\tau_s^\sigma}),
\]
\endgroup\ignorespacesafterend
{with zero for no insertion.}
{Hindsight uses}
\begingroup
\[
B_s^{\mathrm{HS}}=\max\{0,\max_{t\leq T_s,i}\gamma^{t-1}b_i\widehat q_i(h_{s,t})\}.
\]
\endgroup\ignorespacesafterend
{For identical records, bids, and quality predictions, $B_s^\sigma\leq B_s^{\mathrm{HS}}$ pathwise, and the inequality also holds after averaging.}
{This is not a bound on revenue, true-click welfare, or a causal optimizer's attainable performance.}
{The descriptive gap fraction is}
\begingroup
\[
\frac{\overline B_{\mathrm{StopNet}}-\overline B_{\mathrm{First}}}
{\overline B_{\mathrm{HS}}-\overline B_{\mathrm{First}}}=27.80\%.
\]
\endgroup\ignorespacesafterend

\paragraph{{Uncertainty.}}
{For StopNet versus first-turn comparisons, we use 4,000 paired bootstrap resamples of test records with seed 42, retaining both policies' outcomes for each sampled record. The timing-by-renderer comparisons use 2,000 paired bootstrap resamples.}
{A sensitivity analysis resamples the 1,923 user-profile clusters with their associated records.}
{Cluster-based difference intervals are $[0.008723,0.018042]$ for critical revenue, $[0.010572,0.020191]$ for envelope proxy revenue, and $[0.005506,0.013191]$ for predicted welfare.}
{All compare StopNet with first-turn insertion.}
{These intervals condition on the fitted models and replay protocol; they do not cover training randomness or proxy-model error.}
{Horizon-stratum comparisons are exploratory without multiplicity adjustment.}

\subsection{Training Hyperparameters}

{The quality predictor combines Qwen3-Embedding-0.6B with a two-layer MLP trained by binary cross-entropy on history, advertisement, and click tuples. The main-table StopNet runs use seeds 42, 43, and 44 and hold the quality predictor, judge, and renderer fixed. For each seed, metrics first average over all test records; reported standard deviations then use the three seed-level means with denominator $3-1$. Evaluation and Bellman-target bid-sampling seeds remain fixed at 42. The detailed diagnostic runs retain the seed-42 checkpoint.}
{The frozen renderer is the second-round Qwen3-4B SFT checkpoint used in the main results.}
{The renderer is fine-tuned on chat-formatted advertising trajectories. The data split remains fixed across StopNet seeds, and every checkpoint is evaluated across all record lengths.}
{Splits are by session id, so turns of the same dialogue never cross the train/test boundary.}
{StopNet has two ReLU hidden layers of width 128, dropout 0.05, and two outputs; 12 current-state features are standardized on the training set.}

{Training uses MSE, AdamW with learning rate $10^{-3}$ and weight decay $10^{-4}$, batch size 256, at most 30 epochs, and early-stopping patience {2} based on training-set MSE. {The launcher sets patience to 1, and the StopNet training call multiplies it by 2.}}

{The reported run stops at epoch 9.}
{Each evaluation stratum loads the same checkpoint.}
{Formal response generation uses vLLM; the renderer is not retrained in this run.} {The launch script sets temperature to zero and limits output to 128 new tokens.}
{Discount and experience penalty are $\gamma=0.95$ and $\lambda=0.25$.}

\paragraph{{StopNet inputs and execution.}}
{The 12 features are turn index divided by five; largest score, second-largest score, and their gap; quality of the top-scored candidate; maximum, mean, and standard deviation of qualities; bid of the top-scored candidate; maximum and mean bids; and a payment placeholder fixed to zero in this state-construction path. Candidate-level features use the code's descending score sort; allocation uses the first maximum in candidate order. No total-record-length feature is passed to StopNet. Training rewards are the maximum identity-weighted predicted score, without response rendering or the experience penalty. The Qwen3-Embedding-0.6B encoder and two-layer MLP provide pre-render quality predictions. A separate User-LLM judge scores generated responses for clicks and retention; it does not reuse the quality-model weights.}

\paragraph{{Baseline replay details.}}
{Random-Round samples uniformly from available recorded turns and holds the selected turn fixed during payment replay. Myopic-Threshold replaces only the stop/wait action-value comparison with its preset score threshold, providing an adaptive control without learned continuation values.}

\paragraph{{Additional fixed-turn control.}}
{The stratified evaluator also runs second-turn insertion for three-turn records. It omits a duplicate second-turn policy for two-turn records, where it equals session-end insertion; one-turn records have no second turn. No additional pooled second-turn result is inferred from this implementation detail.}

\paragraph{{Reproducibility scope.}}
{Recomputing reported aggregates requires the per-record results and analysis scripts; rerunning inference additionally requires the frozen quality, StopNet, and renderer checkpoints and matching configuration.}
{Retraining and regenerating the synthetic corpus require the original data, splits, and generation services.}
{Complete model-version and user-mapping provenance is not established by the scripts alone, so byte-identical regeneration is not claimed.}

\subsection{Payment Computation}
\label{app:payment-conventions}

{The main implementation holds rival bids and the record fixed and binary-searches the original winner's report using the predicate that it still wins after rerunning stopping and allocation.}
{It searches from zero to the original bid with 24 bisection iterations.}
{Searching a binary winner indicator presumes an upper-interval winning set; nonmonotone stopping can violate that premise.}
{The payment audit identifies two records with winning reports below the returned value.}
{Finite probing establishes those counterexamples without certifying the remaining continuous report ranges.}

{For the envelope comparison, a 64-point uniform bid grid on the type interval $[0.1,4.1]$ is augmented with the original bid {and critical CPC, retaining added probes only when they lie in the grid interval}.}

{Each probe reuses the same frozen record; it does not average fresh independent continuation rollouts.}
{Let $a_s(z)$ be that advertiser's discounted predicted-click allocation on the record when reporting $z$.}
{Left-rectangle integration gives}
\begingroup
\[
\begin{aligned}
\widetilde m_s(b)&=b\,a_s(b)-\int_{\underline\theta}^{b}a_s(z)\,dz,\\
\widetilde P_s(b)&=\widetilde m_s(b)/a_s(b)
\end{aligned}
\]
\endgroup\ignorespacesafterend
{when $a_s(b)>10^{-10}$, with zero CPC otherwise; the displayed integral denotes the numerical grid approximation here.}

{It is exact for the grid's piecewise-constant interpolation, not necessarily for the underlying continuous-report allocation.}

{We retain signed outputs when diagnosing negative envelope payments.}
{Replacing the main-table critical price with this CPC and keeping the response-sensitive click score gives the envelope proxy-revenue comparison.}
{Using predicted clicks instead gives mean revenues 0.708947 for StopNet and 0.694091 for first-turn insertion (2.14\%; paired difference CI $[0.011094,0.019044]$).}

{For the original winner in each record, grid regret compares truthful predicted utility with the best grid report under numerical envelope transfers; it does not search all advertisers.}
{The regret routine uses $m(r)=r a_s(r)-I_s(r)$ at every report, where $I_s$ is the left-rectangle integral. If $a_s(r)=0$ after positive earlier allocation, this formal transfer can be negative. In contrast, the CPC routine returns zero payment when allocation is at most $10^{-10}$. Consequently reported regret concerns the formal numerical envelope, not the executable zero-allocation CPC convention, and neither diagnostic establishes incentives under the main-table critical price.}
{Monotonicity uses tolerance $10^{-6}$.}
{Nonmonotonicity and positive regret are separate diagnostics:}

\begingroup
\begin{table}[t]
\captionsetup{font=10pt}
\centering\small
\setlength{\tabcolsep}{4pt}
\begin{tabular}{lrrr}
\toprule
Horizon & Nonmonotone & Positive regret & Negative pay\\
\midrule
1 & 0 & 0 & 0\\
2 & 70 & 69 & 1\\
3 & 49 & 46 & 0\\
\midrule
All & 119 & 115 & 1\\
\bottomrule
\end{tabular}
\caption{Counts among 2,741 frozen-record diagnostics. Positive regret exceeds $10^{-9}$. The table concerns envelope replay, not the main-table critical payment.}
\label{tab:diagnostic-counts}
\end{table}
\endgroup\ignorespacesafterend
{Of the records, 95.80\% have numerical-zero regret at this tolerance.}
{Mean regret is 0.006245; the 95th percentile is approximately zero, the 99th is 0.217219, and the maximum is 0.726662.}
{These finite-grid, model-based quantities do not certify DSIC in expectation over future paths.}

\paragraph{Additional allocation and payment diagnostics.}
{On the 119 nonmonotone records, the $L_2$ isotonic projection of the probed curve leaves a sup-norm residual with mean $0.150$ (median $0.129$, max $0.465$); the minimum $L_\infty$ distance to a nondecreasing sequence averages $0.100$.}
{For ordered grid values $a_1,\ldots,a_K$, this minimum distance is $\frac12\max_{j<k}(a_j-a_k)_+$. It allows adjustments both above and below the observed curve; it is not a distance to a pointwise majorant.}
{A sufficient equality case is $a(z)=c\mathbf{1}\{z\ge\kappa\}$, with $c>0$ and $\underline\theta\le\kappa\le b$. Then $\int_{\underline\theta}^{b}a(z)\,dz=c(b-\kappa)$, so envelope CPC equals $\kappa$. This is a sufficient condition, not an if-and-only-if characterization.}
{Across all 2,741 records, the absolute difference between saved envelope CPC and critical-report CPC is at most $10^{-6}$ for 2,491 records (90.88\%). This is an empirical equality count across all allocation shapes, not a count of records certified to have a single allocation jump. Payment equality and allocation monotonicity are different properties; the numerical agreement does not establish a valid critical threshold or incentive compatibility.}
{The platform observes the submitted bid, not the private value. Both payments are computable from reports and replays. Nonmonotone curves can yield negative envelope CPC; binary-search CPC remains nonnegative but need not identify a valid critical threshold.}
{For the allocation illustration, six grid-monotone and two nonmonotone examples are sampled separately with seed 42; this selection does not represent class frequencies. For the utility illustration, a grid-monotone record is sampled with seed 42 and compared with the largest-regret record. Utility gains are computed at the saved probes, with the original submitted bid treated as the assumed value. Stars mark the original report and the open circle a best grid deviation; connecting lines do not represent additional evaluations.}

{The 2,588 records with the same realized turn and winner as first-turn insertion contribute 40.17\% of the aggregate critical-revenue difference. This is an accounting decomposition, not a pure causal timing effect.}
{The fixed-price sensitivity holds the original winner and its original CPC fixed while choosing its best predicted-click outcome over the grid.}
{Its predicted-click revenue changes from 0.706722339 to 0.710207793.}
{For nonnegative fixed prices, inclusion of the original bid already makes nondecrease structurally possible by maximizing clicks.}
{This check is therefore distinct from re-running payments for profitable deviations or evaluating a strategic equilibrium.}
{Under a common report-independent path law and the same true click allocation, exact integration can commute with expectation.}
{This may recover the expected envelope transfer from pathwise transfers, but does not ensure nonnegative prices or per-click IR on every path.}
{Predicted clicks, finite quadrature, and zero-allocation conventions must therefore be distinguished from the exact theorem.}

{Figure~\ref{fig:ic-by-trigger} groups the grid diagnostics by actual StopNet insertion turn, not record length.}
\begin{figure}[t]
\centering
\includegraphics[width=\columnwidth]{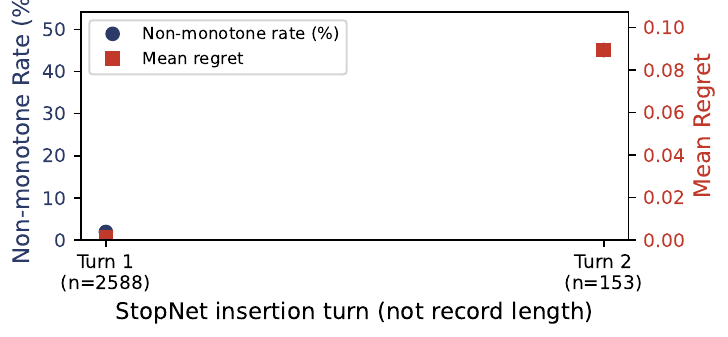}
\caption{{Frozen-record diagnostics grouped by StopNet insertion turn, not record length. Unconnected points report the nonmonotone rate and mean formal numerical-envelope regret; group sizes appear on the horizontal axis.}}
\label{fig:ic-by-trigger}
\end{figure}

\subsection{Prompts}
{The following prompts document historical label generation and renderer preparation.}

{Neither template includes the advertiser bid.}
{Current evaluation uses a separate User-LLM judge for click and retention scores, independently of the pCTR model $G_\eta$; human calibration has not been established.}

\paragraph{User-simulation judge.}

{Figure~\ref{fig:judge-prompt} reproduces the original label-generation template for click and trust/experience scoring. Current response evaluation uses the separate User-LLM judge described above; this template alone does not document its exact runtime model version or configuration.}
\begin{figure*}[t]
\small\ttfamily\raggedright
\textbf{[System]}\\
You are simulating a realistic e-commerce user in a multi-turn AI shopping conversation. You must judge whether the assistant's latest response would make the user click the inserted sponsored product and whether the response preserves user trust.\\[2pt]
Use only the user's stated needs and the dialogue context. Do not reward a product merely because it has a high bid. Penalize ads that are premature, irrelevant, forced, or distracting. Reward ads that directly answer the user's current need and are integrated naturally.\\[2pt]
Return strict JSON only: \{clicked, click\_probability, gamma\_score, reason\}, where gamma\_score is the retention/trust score after seeing the response (1 = satisfied and willing to continue, 0 = likely to leave).\\[4pt]
\textbf{[Input]}\\
User's original query: \{query\}\\
Dialogue history before the assistant's current response: \{history\}\\
Sponsored product inserted by the assistant: title \{title\}, ad id \{id\}\\
Assistant's current response: \{response\}\\
Judge the response from the user's perspective and return strict JSON only.
\caption{User-simulation judge system prompt and input template. The advertiser bid is not included in the input.}
\label{fig:judge-prompt}
\end{figure*}
The judge never sees the advertiser bid.

\paragraph{Renderer (SFT training template).}
Figure~\ref{fig:renderer-prompt} reproduces the renderer SFT system prompt, where \texttt{<title>} and \texttt{<id>} are the selected ad's title and id.
\begin{figure*}[t]
\small\ttfamily\raggedright
You are a helpful AI shopping assistant with advertising capabilities.\\[2pt]
\#\# Your Role\\
You help users with shopping questions while naturally integrating a sponsored product that the platform has selected through an auction. Your goal is to maximize user satisfaction while seamlessly mentioning the sponsored product.\\[2pt]
\#\# Sponsored Product\\
Product: <title>\\
Ad ID: <id>\\[2pt]
\#\# Rules\\
1. First and foremost, answer the user's question helpfully and accurately.\\
2. Integrate the sponsored product exactly ONCE in your response.\\
3. Use the strict format: @<title>@[<id>] --- do not modify the ad title.\\
\hspace*{1em}Correct:\ \ @<title>@[<id>]\\
\hspace*{1em}Wrong:\ \ \ \ @<title in lowercase>@[<id>]\\
\hspace*{1em}Wrong:\ \ \ \ <title> [<id>]\\
\hspace*{1em}Wrong:\ \ \ \ @<title>@(<id>)\\
4. The ad insertion should feel natural --- as if you genuinely recommend the product.\\
5. If the product is directly relevant to the user's question, weave it into your main recommendation naturally.\\
6. If the product is not directly relevant, find a natural transition point:\\
\hspace*{1em}- ``You might also consider...''\\
\hspace*{1em}- ``On a related note...''\\
\hspace*{1em}- ``While we are on the topic of [category]...''\\
7. Do NOT mention auctions, bids, pCTR, payments, advertising, sponsorship, or any internal mechanism.\\
8. Do NOT insert multiple products or use the @...@[...] format more than once.\\
9. Keep your response concise and helpful --- typically 2--4 sentences.\\[2pt]
\#\# Good Example\\
User: What is a good desk lamp for studying?\\
Assistant: For long study sessions, you want adjustable brightness and a wide light bar to reduce eye strain. @<title>@[<id>] is a solid option worth checking out. Pair it with a warm-tone setting at night to reduce blue light exposure.\\[2pt]
\#\# Bad Examples\\
Bad (too forced): You should definitely buy @<title>@[<id>] right now! It is the best product ever!\\
Bad (no ad): For studying, get a lamp with adjustable brightness. Any LED desk lamp will work.\\
Bad (wrong format): Check out <title> [<id>].
\caption{Renderer SFT training system prompt. \texttt{<title>} and \texttt{<id>} are the selected ad's title and id.}
\label{fig:renderer-prompt}
\end{figure*}
The candidate list and bids are stripped from the training text, leaving only the selected product's title and id; only turns that insert an ad are kept as supervised targets.

\subsection{Hindsight Reference Metrics}
\label{app:hindsight-reference}
Table~\ref{tab:hindsight-reference} reports the full Hindsight reference row separately from the main policy comparison. Hindsight uses future dialogue information and bounds predicted welfare only, not revenue or the other reported metrics. Values use the same metric definitions and mean $\pm$ SD format as the main table.

\begin{table*}[t]
\centering
\small
\setlength{\tabcolsep}{4pt}
\resizebox{\textwidth}{!}{%
\begin{tabular}{lrrrrrrrr}
\toprule
Method & Net Rev. & Reward & CTR & Trigger & Bid & CPC & Info. Rent & Bid Value \\
\midrule
Hindsight reference (refined) & $0.770_{\pm 0.000}$ & $0.735_{\pm 0.000}$ & $0.817_{\pm 0.000}$ & $1.135_{\pm 0.000}$ & $1.711_{\pm 0.000}$ & $0.962_{\pm 0.000}$ & $0.600_{\pm 0.000}$ & $1.370_{\pm 0.000}$ \\
\bottomrule
\end{tabular}}
\caption{Performance and economic metrics of the Hindsight reference with the refined dialogue model (mean $\pm$ SD).}
\label{tab:hindsight-reference}
\end{table*}

\subsection{Adaptive Timing Baseline}
\label{app:myopic}

The myopic baseline inserts at the first turn whose current winner score exceeds the fixed threshold $0.4$, without learned continuation values; eligible terminal states insert under the shared terminal rule.

The main-text timing-by-renderer table reports a timing $\times$ renderer factorial under the same critical-payment computation.

Refining the renderer increases net revenue by $4.60\%$ under Myopic and $4.37\%$ under StopNet. Switching from Myopic to StopNet changes net revenue by $-0.99\%$ with the base renderer and $-1.21\%$ with the refined renderer.
The unmatched contrast between refined StopNet and base Myopic is $+3.34\%$ (absolute difference CI $[+.015077,+.031496]$). Because both timing and rendering change, this contrast does not identify a benefit of learned timing.
Predicted welfare favors learned stopping, $+0.86\%$ with CI $[.007532,.015147]$, and is renderer-invariant because $\gamma^{t-1}b_i\widehat q_{\eta,i}$ does not read generated text; the measured renderer effect on it is exactly zero.
Learned stopping raises the winner's mean predicted click quality by $2.15\%$ but lowers mean critical CPC by $1.87\%$. Revenue equals the response-sensitive discounted bid value minus information rent (Section~\ref{app:metrics}), not predicted welfare minus rent. These aggregate comparisons do not by themselves identify a causal payment-channel explanation.
Relative to the Hindsight predicted-welfare reference, StopNet closes $27.80\%$ of the gap from first-turn insertion, whereas Myopic lies below first-turn insertion (gap fraction $-6.33\%$). The analogous revenue-gap fractions are $20.70\%$ for StopNet and $34.50\%$ for Myopic, but are only descriptive: Hindsight is not a revenue upper bound.
Myopic and StopNet mean insertion turns are 1.110544 and 1.055819, respectively, with turn-one/two/three counts 2,510/159/72 and 2,588/153/0.
Three-turn records account for 80.2\% of the seed-42 net-revenue gap to matched Myopic. This concentration motivates examining longer records, but does not identify mixed-horizon training as its cause.

Across the three StopNet seeds the refined-renderer revenues are 0.719922, 0.724342, and 0.729032, so the matched Myopic value 0.728724 lies inside this range and below the best seed; the $-1.21\%$ gap is a seed-42 paired quantity and the gap to the three-seed mean is $0.59\%$.

\begingroup
\subsection{Payment, Welfare, and Horizon Comparisons}
\endgroup\ignorespacesafterend
{Table~\ref{tab:payment-welfare} compares payment and welfare metrics; Table~\ref{tab:horizon} reports results by recorded horizon.}

\begingroup
\begin{table}[t]
\captionsetup{font=10pt}
\centering\small
\begin{tabular}{lrr}
\toprule
Evaluation & First-turn & StopNet\\
\midrule
Critical $\times$ judge revenue & .706720 & .719922\\
Envelope $\times$ judge revenue & .706989 & .722342\\
Predicted welfare & 1.317009 & 1.326171\\
\bottomrule
\end{tabular}
\caption{Seed-42 refined-renderer comparisons on 2,741 records. Hindsight predicted welfare is 1.349969; payment revenue is a separate metric.}
\label{tab:payment-welfare}
\end{table}
\endgroup\ignorespacesafterend

\begingroup
\begin{table}[t]
\captionsetup{font=10pt}
\centering\small
\setlength{\tabcolsep}{4pt}
\begin{tabular}{lrrr}
\toprule
Length & Records & Revenue gain & Difference 95\% CI\\
\midrule
1 & 208 & 0.00\% & $[0,0]$\\
2 & 1,825 & 2.54\% & $[.012937,.025146]$\\
3 & 708 & 0.45\% & $[-.005600,.011109]$\\
\bottomrule
\end{tabular}
\caption{Seed-42 StopNet versus first-turn insertion with refined rendering and critical-report payments. Intervals are absolute paired revenue differences; length-stratum comparisons are exploratory.}
\label{tab:horizon}
\end{table}
\endgroup\ignorespacesafterend
{For records of length at least two, the critical-revenue gain is 2.05\%, with absolute difference CI $[0.009501,0.019592]$.}
{Both critical and envelope proxy-revenue difference intervals include zero in the three-turn stratum.}
{The one-turn stratum is forced by its horizon and provides no learned timing comparison.}

{The three-seed summary flags minor evaluation differences for the base-renderer first-, last-, and random-turn controls in seed 44. Each control\textquotesingle s mean revenue changes by approximately $0.000002$ relative to seed 42. Refined-renderer control metrics remain unchanged. Reported standard deviations retain these differences rather than forcing them to zero.}

\subsection{Case Study: Autonomous Insertion at Turn Two}
\label{app:case-autonomous-stop}

Figure~\ref{fig:case-autonomous-stop} illustrates autonomous insertion on the simulated test record \texttt{sess\_h\_0305\_q01}. The recorded horizon is three turns. StopNet waits at turn one and selects insertion at turn two, before the terminal rule applies. Replaying the saved checkpoints reproduces the evaluated stopping turn, winning advertisement, bid, and winner pCTR.

The user initially asks how to upgrade their kitchen setup. After a clarifying question, they specify one pan for eggs, vegetables, and chicken. The candidate advertisements are an ASUS laptop, a Lodge cast iron skillet, and a Zwilling cookware set. The mechanism's predicted click quality for the Lodge skillet rises from $0.3261$ to $0.3932$. At turn one, the predicted stop action value is $0.7443$, below the wait action value of $0.9476$. At turn two, this comparison reverses: the stop value is $0.8960$, exceeding the wait value of $0.2226$. StopNet therefore selects STOP, and the Lodge skillet wins with bid $2.19$ and implemented critical-report CPC $0.9208$. These action values are StopNet outputs, distinct from the candidates' bid-weighted click scores.

\begin{figure}[htbp]
\centering
\includegraphics[width=\columnwidth]{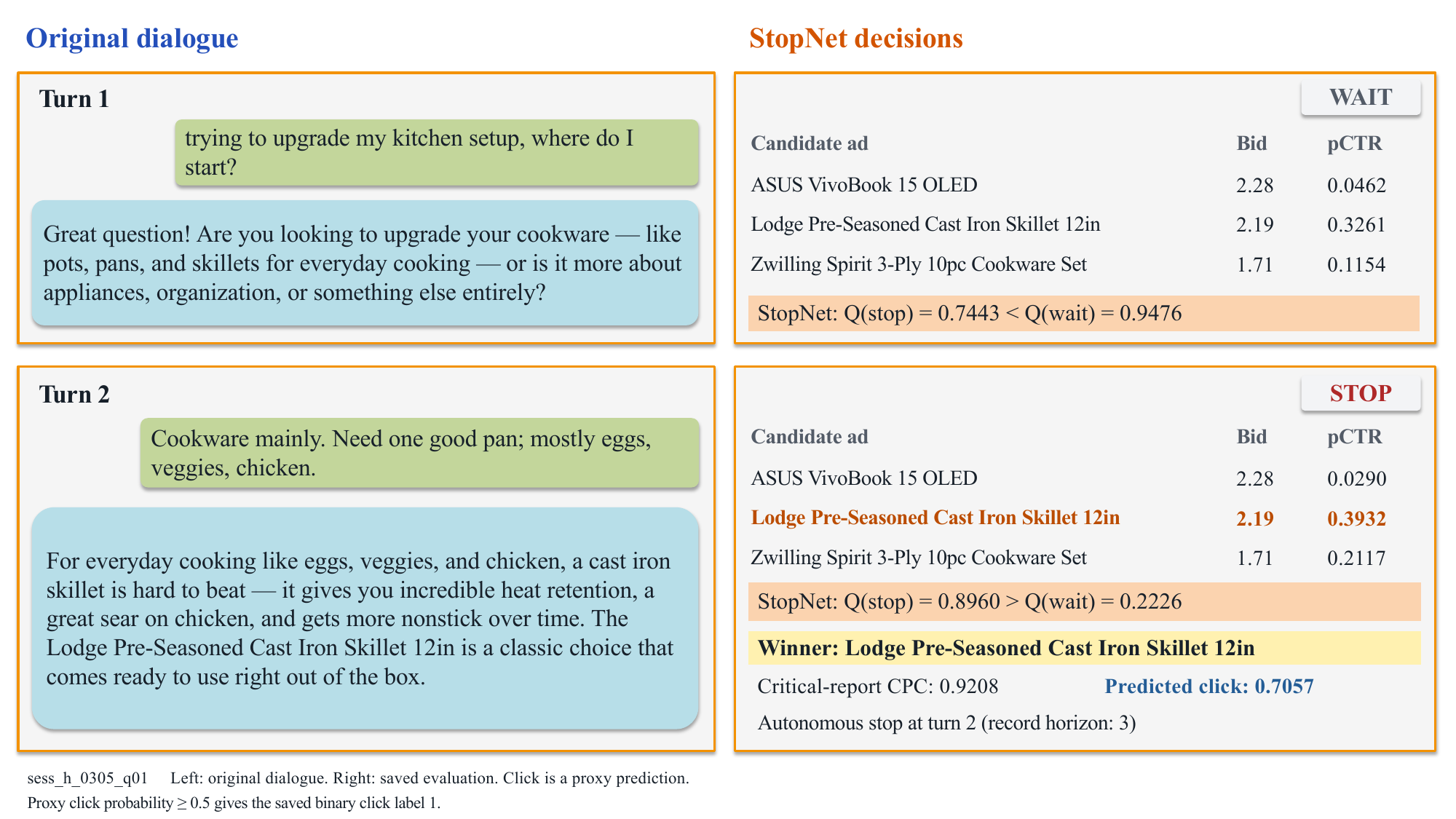}
\caption{Original dialogue and replayed autonomous insertion at turn two. Click denotes a proxy prediction.}
\label{fig:case-autonomous-stop}
\end{figure}

The evaluation proxy assigns the saved rendered response a click probability of $0.7057$, yielding binary label $1$ under the $0.5$ threshold. This probability differs from the mechanism pCTR of $0.3932$: the former evaluates the rendered response, whereas the latter guides allocation. The figure quotes the first two original dialogue turns in full; these utterances are not the renderer output used to compute the evaluation score. The unused third recorded turn and its click label do not describe this insertion. This selected example illustrates the learned stopping decision and its predicted outcome, rather than establishing observed user clicks, optimal timing, or a causal revenue gain.

\end{document}